\documentclass{article} 
\usepackage{iclr2026_conference,times}

\usepackage{amsmath,amsfonts,bm}
\usepackage{accents}

\def\eqref#1{equation~\ref{#1}}

\def\1{\bm{1}}

\def\vc{{\bm{c}}}

\def\vk{{\bm{k}}}

\def\vt{{\bm{t}}}

\def\vx{{\bm{x}}}
\def\vy{{\bm{y}}}
\def\vz{{\bm{z}}}

\def\mI{{\bm{I}}}

\DeclareMathAlphabet{\mathsfit}{\encodingdefault}{\sfdefault}{m}{sl}
\SetMathAlphabet{\mathsfit}{bold}{\encodingdefault}{\sfdefault}{bx}{n}

\def\gA{{\mathcal{A}}}

\def\gX{{\mathcal{X}}}
\def\gY{{\mathcal{Y}}}

\DeclareMathOperator*{\argmax}{arg\,max}

\newcommand{\autorefLemma}[1]{%
  \hyperref[#1]{Lemma~\ref*{#1}}
}
\newcommand{\barbelow}[1]{\underaccent{\bar}{#1}}

\newcommand{\ours}{T1\xspace}

\usepackage[utf8]{inputenc} 
\usepackage[T1]{fontenc}    
\usepackage{hyperref}
\usepackage{url}
\usepackage{booktabs}       
\usepackage{amsfonts}       
\usepackage{nicefrac}       
\usepackage{microtype}      
\usepackage{xcolor}         
\usepackage{booktabs}
\usepackage{multirow}
\usepackage{graphicx}
\usepackage{makecell}
\usepackage{adjustbox}
\usepackage{siunitx} 
\usepackage{colortbl}
\usepackage{subcaption}
\usepackage{algorithm}      
\usepackage{algpseudocode}  
\usepackage{bm}
\usepackage{xspace}
\usepackage{amsmath}
\usepackage[nameinlink]{cleveref}
\usepackage{enumitem}

\usepackage{xcolor}
\usepackage[utf8]{inputenc} 
\usepackage[T1]{fontenc}    
\usepackage{hyperref}       
\usepackage{url}            
\usepackage{booktabs}       
\usepackage{amsfonts}       
\usepackage{nicefrac}  
\usepackage{amsthm}
\usepackage{microtype}      
\usepackage{xcolor}         
\usepackage[table]{xcolor}
\usepackage{enumitem}
\usepackage{graphicx}
\usepackage{multirow}
\usepackage{makecell}
\usepackage{caption}
\usepackage{subcaption}
\usepackage[flushleft]{threeparttable}
\usepackage{mathtools}
\usepackage[nameinlink]{cleveref}
\usepackage{soul}
\usepackage{arydshln}
\usepackage{stackengine}
\usepackage{wrapfig}
\usepackage{makecell}
\usepackage{colortbl}
\usepackage{arydshln}
\usepackage[most]{tcolorbox}
\usepackage{listings}
\usepackage{xspace}
\usepackage{boxedminipage}

\definecolor{Red}{rgb}{0.768, 0.054, 0.054}
\definecolor{Blue}{rgb}{0.152, 0.294, 0.925}
\definecolor{Green}{rgb}{0,0.4,0.7}
\hypersetup{
    colorlinks=true,
    citecolor=teal,
    linkcolor=Red,
    urlcolor=Green,
}

\newtcolorbox[auto counter, number within=section]{promptbox}[2][]{
  colback=white,
  colframe=blue!80!black,
  coltitle=white,
  title=Prompt~\thetcbcounter: #2,
  fonttitle=\bfseries\normalsize,
  boxrule=1pt,
  arc=2mm,
  top=2mm,
  bottom=2mm,
  width=0.95\textwidth,
  #1
}

\newtcolorbox[auto counter, number within=section]{examplebox}[2][]{
  colback=white,
  colframe=black,
  breakable,    
  boxrule=0.5pt,
  sharp corners,
  left=4pt,
  right=4pt,
  top=4pt,
  bottom=4pt,
  enhanced,
  title=Example~\thetcbcounter: #2,
  #1
}

\newtheorem{theorem}{Theorem}[section]

\newtheorem{lemma}[theorem]{Lemma}

\title{T1: Tool-integrated Verification for\\Test-time Compute Scaling in\\Small Language Models}

\iclrfinalcopy

\author{Minki Kang\textsuperscript{1}\thanks{Equal contribution. Work done at KRAFTON. Contact: \texttt{minkikang@kaist.ac.kr}}
\quad
Jongwon Jeong\textsuperscript{2}$^*$
\quad
Jaewoong Cho\textsuperscript{3}\\
\textsuperscript{1}KAIST
\quad\textsuperscript{2}University of Wisconsin-Madison
\quad\textsuperscript{3}KRAFTON
}

\begin{document}

\maketitle

\begin{abstract}
Recent studies have demonstrated that test-time compute scaling effectively improves the performance of small language models (sLMs).
However, prior research has mainly examined test-time compute scaling with an additional larger model as a verifier, leaving verification by sLMs underexplored.
In this work, we investigate whether sLMs can reliably verify the output candidates under test-time scaling.
We find that even with knowledge distillation from larger verifiers, sLMs struggle with verification tasks requiring memorization, such as numerical calculations and fact-checking.
To address this limitation, we propose \textbf{Tool-integrated verification (\ours)}, a two-stage framework that first filters candidates with external tools and then uses an sLM for final verification, offloading memorization-heavy steps to tools such as a code interpreter.
Within \ours, we prove that offloading to external tools reduces the memorization burden on sLMs and improves test-time scaling performance.
Experiments on the MATH benchmark demonstrate that, with \ours, a Llama-3.2 1B model under test-time scaling outperforms the significantly larger Llama-3.1 8B model. 
Moreover, \ours improves the verification accuracy of both process reward models (PRMs) and critic models.
Our findings highlight the potential of tool integration to substantially improve the verification abilities of sLMs.
\end{abstract}
\section{Introduction}
\label{sec:intro}

Recent advances in large language models (LLMs) have demonstrated strong emergent abilities through large-scale pretraining~\citep{GPT3, GPT4o, Gemini}, enabling them to tackle complex reasoning tasks such as mathematical problem-solving and competitive coding~\citep{CoT, VerifyStepbyStep}.
While small language models (sLMs) offer advantages in deployment efficiency and cost~\citep{MobileLLM, sLMsurvey}, they struggle significantly with high-complexity tasks~\citep{emergent}.

\emph{Test-time compute scaling} has emerged as a promising approach to enhance sLMs by dynamically allocating additional computation during inference~\citep{InferenceScalingLaw}. Prior works suggest that test-time scaling can surpass pretraining-based scaling~\citep{ScalingTestTime}, allowing a 3B LM to outperform a 405B LLM on mathematical benchmarks such as MATH and AIME~\citep{1BSurpass}. This success depends on reliable verification of generated solutions.

To enable reliable verification, existing approaches have leveraged process reward models (PRMs) and critic models~\citep{MathShepherd, GenRM}, but these typically require LLMs (7B+ parameters).
Relying on large verifiers counteracts the efficiency benefits of sLMs. 
Therefore, it remains unclear whether \emph{verification}, where sLMs verify the generated solutions, can enable strong reasoning capabilities without relying on larger models.
This raises a key research question:
\begin{quote}
\vspace{-0.05in}
    \textit{Can small language models reliably perform verification for test-time scaling?}
\vspace{-0.05in}
\end{quote}

While verification is often easier than generation, prior work has shown that sLMs still struggle to verify the solutions~\citep{mindthegap}. 
Our concept-proof experiment in \autoref{fig:concept}~(b) confirms this finding that larger models can reliably verify with chain-of-thought reasoning alone. In contrast, sLMs fail to verify even simple calculations, particularly as the complexity of calculations N increases.
We hypothesize that this gap is due to the limited capacity of sLMs to memorize all calculation facts required for verification \citep{LongtailKnowledge}.

However, as shown in \autoref{fig:concept}~(b), code generation and execution substantially improves sLM's verification accuracy, narrowing the gap with larger models even as $N$ increases.
This result suggests that integrating external tools with sLMs reduce the need to memorize arithmetic facts.
Therefore, tool integration is not merely beneficial but \emph{necessary} to enable successful verification in sLMs.

\begin{figure*}
    \centering
    \includegraphics[width=0.9\linewidth]{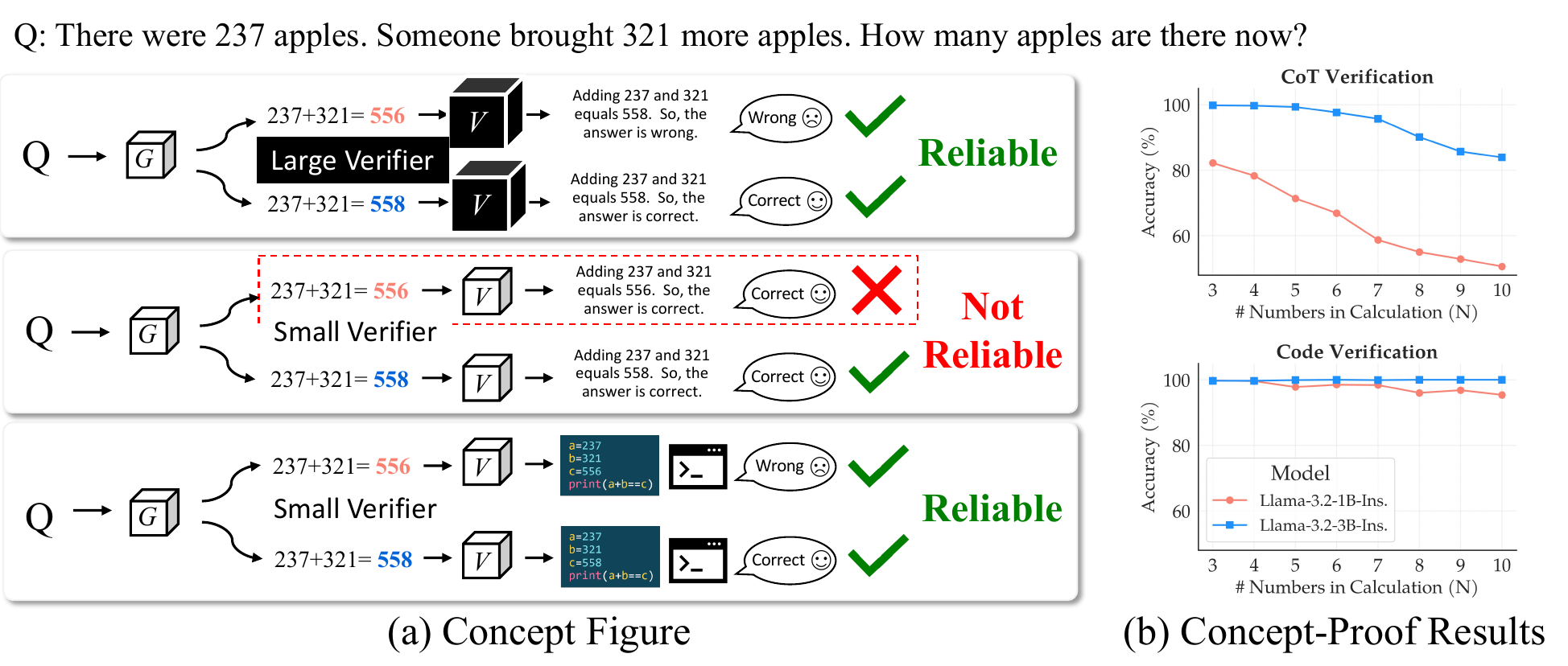}
    \vspace{-0.15in}
    \caption{\textbf{(a) Concept.} Small language models (sLMs) often fail due to their limited capacity. However, when sLMs utilize external tools, their reliability improves. \textbf{(b) Concept-Proof Experimental Results.} We evaluate Llama 1B and 3B models on verifying calculations of $N$ three-digit numbers. The 1B model's performance drops significantly as $N$ increases, while the 3B model remains stable. Enabling code generation and execution largely mitigates this drop for the 1B model.}
    \vspace{-0.15in}
    \label{fig:concept}
\end{figure*}

Motivated by these findings, we introduce \textbf{Tool-integrated Verification (T1)}, a two-stage framework in which external tools first verify candidate solutions and sLMs then verify the filtered solutions.
By offloading memorization-heavy steps, such as numerical calculations and fack-checking, to external tools, \ours enables sLMs to achieve verification accuracy comparable to much larger models without increasing parameters.
Crucially, this two-stage design applies seamlessly to both generative verifiers and process reward models (PRMs), offering a single framework for test-time scaling.
Moreover, the two-stage design theoretically guarantees that tool-based verification reduces the memorization burden and improves test-time scaling performance.

Our experiments demonstrate that \ours enables both generative verifiers~\citep{GenRM} and PRMs for more accurate verification.
This leads to significant performance improvements on widely-used mathematical and multi-domain reasoning benchmarks, with notable gains on GSM8K~\citep{GSM8K} and MATH~\citep{MATH}.
These results underscore the effectiveness of T1 in improving the performance of sLMs under test-time compute scaling.

Our contributions are as follows:
\begin{itemize}[itemsep=0.8mm, parsep=1pt, leftmargin=*]
\vspace{-0.1in}
    \item We conduct a systematic study of sLMs’ verification under test-time scaling, identify memorization-heavy steps in verification as a key bottleneck, and motivate addressing them with external tools.

    \item We propose Tool-integrated Verification (\ours), a two-stage framework that leverages external tools to offload memorization-heavy steps before sLM verification.

    \item We provide a theoretical analysis showing that external tools reduces the memorization burden and that two-stage design in \ours improves test-time scaling performance.

    \item We show that T1 integrates seamlessly with both generative verifiers and PRMs, achieving strong results on math reasoning benchmarks GSM8K and MATH.
\end{itemize}
\section{Related Works}
\label{sec:related-works}

\subsection{Test-time compute scaling}
Test-time compute scaling has emerged as a promising approach for improving the reasoning capabilities of large language models (LLMs)~\citep{InferenceScalingLaw}. It can be broadly categorized into \emph{sequential} and \emph{parallel} methods~\citep{ScalingTestTime}. Sequential scaling iteratively refines solutions by leveraging  post-training, enabling the model to perform self-reflection and verification \citep{s1, R1}. 
Parallel scaling, in contrast, generates multiple candidate solutions simultaneously and selects the best one using a verifier model \citep{GSM8K, VerifyStepbyStep, LargeLanguageMonkeys}. A common strategy is the \emph{best-of-N} method, which produces $N$ parallel outputs and ranks them based on verification scores \citep{GSM8K, VerifyStepbyStep}. Increasing $N$ has been shown to enhance LLM on challenging benchmarks \citep{LargeLanguageMonkeys, ScalingTestTime}.

In this work, we focus on the \emph{parallel scaling} paradigm due to its simplicity and popularity. Prior research shows that even small models can achieve strong results when paired with a large verifier in parallel scaling \citep{1BSurpass}. 
We further investigate whether small language models (sLMs) can verify, enabling test-time scaling without large models.

\subsection{Verifier in test-time compute scaling}
The verifier plays a crucial role in parallel scaling.  One approach, the Process Reward Model (PRM), scoring each reasoning step individually using a regression head, enables fine-grained feedback~\citep{VerifyStepbyStep, MathShepherd, VersaPRM}.
An alternative approach leverages an LLM itself as a \emph{Critic} model, prompting it to evaluate reasoning steps \citep{llmjudge}. \citet{ProcessBench} have shown that powerful LLM-based critic models can outperform PRM, particularly in mathematical reasoning tasks.
Recent works~\citep{GenRM, GenerativeRewardModel} proposed the Generative Reward Model (GenRM), formulating verification as a next-token prediction problem with chain-of-thought~\citep{CoT} improving interpretability in step-wise verification.

Despite these advances, ensuring consistent and high-quality step-wise verification remains an open problem for both PRM and GenRM. Additionally, prior works have not thoroughly explored the change in verification performance depending on the size of LMs.

\subsection{Tool-integrated language model}

The integration of external tools has significantly enhanced LLM capabilities.
Program-aided language models~\citep{PAL} introduced delegating computations to interpreters via synthesized code. 
Subsequent works expanded this by using tools like search engines and calculators for fact retrieval and arithmetic~\citep{Toolformer, ToolLLM}.
Recent methods further integrate tools into multi-step reasoning~\citep{ToRA, PaD}, reward modeling~\citep{ToolRM}, and self-correction~\citep{CRITIC}.

Our work extends this line of research by studying \textbf{how to design tool integration that benefits sLM verification across both process reward models (PRMs) and critic models}.
We formulate tool use as an additional dimension of test-time scaling, emphasizing its effectiveness in memorization-heavy verification tasks.
\section{Preliminaries} \label{sec:preliminaries}
\paragraph{Test-time scaling} 
Following~\citet{ScalingTestTime}, we view test-time scaling as modifying the model's proposal distribution.
Given a problem $\vx \in \gX$, we sample a solution $\vy$ from the policy $\pi(\vy \mid \vx, \mI_p; \theta)$ in the set $\gY$, where $\mI_p$ is the generator-specific instruction prompt, and the policy is parameterized by $\theta$ which refers to the pre-trained language models.

Several algorithms can be used to scale test-time computation, including those that adjust the input level (e.g., self-reflection~\citep{SelfReflection}) and those that modify the output level (e.g., best-of-N~\citep{GSM8K}, beam search~\citep{ToT}).
Among them, the best-of-N algorithm is a simple yet powerful approach. It samples multiple candidate solutions from the policy and selects the one with the highest score, as determined by the verifier, as the final prediction. Formally, the Best-of-N policy $\pi^{N}\left(\vy \mid \vx, \mI_p; \theta\right)$ is defined as:
\begin{equation} 
\argmax_{\vy \in \{\vy_1, \ldots, \vy_N\}} r(\vx, \vy), \quad \text{s.t.} \quad \vy_i \sim \pi(\vy \mid \vx, \mI_p; \theta),
\label{eqn:best-of-n}
\end{equation}
where $r(\vx, \vy): \gX \times \gY \to \mathbb{R}$ is a verifier that assigns a scalar score $r$.

\paragraph{Verifier} 
The verifier can be modeled using the following models: (1) process reward model (PRM), which assigns the score of each step of reasoning~\citep{MathShepherd}, (2) critique model, which generates a rationale for the verifiation~\citep{llmjudge}.
For both cases, the sequence of verification scores or tokens $\vz$ are sampled from $\pi(\vz \mid \vx, \vy, \mI_v; \theta)$ where $\mI_v$ is the verifier-specific instruction prompt. 
In general, we use the last score or token of the sequence as the final score for the solution~\citep{ScalingTestTime, GenRM}.
For instance, in generative verifier~\citep{GenRM}, the verification score can be obtained as follows:
\begin{equation}
r(\vx, \vy) = \pi(z_T = \text{`Yes'} \mid \vx, \vy, \mI_v, \vz_{1:T-1}; \theta), \quad \text{s.t.} \quad \vz_{1:T-1} \sim \pi(\vz \mid \vx, \vy, \mI_v; \theta),
\end{equation}
where $\vz_{1:T-1}$ is chain-of-thought~\citep{CoT} and last token $z_T \in \{ \text{`Yes'}, \text{`No'} \}$.

\section{Method}
\subsection{Tool-integrated verification}\label{sec:tool_verification}
Test-time scaling can improve a base policy model that generates valid solutions from its proposal distribution, but the effectiveness of scaling at test-time heavily relies on the performance of the verifier.
However, sLMs often struggle to reliably verify the correctness of their generated outputs~\citep{mindthegap}. Specifically, sLMs exhibit limitations in precisely validating numerical computation or detecting incorrect or outdated knowledge information, due to their limited parameter size and insufficient memorization capacity.

\begin{figure*}
    \centering
    \includegraphics[width=1.0\linewidth]{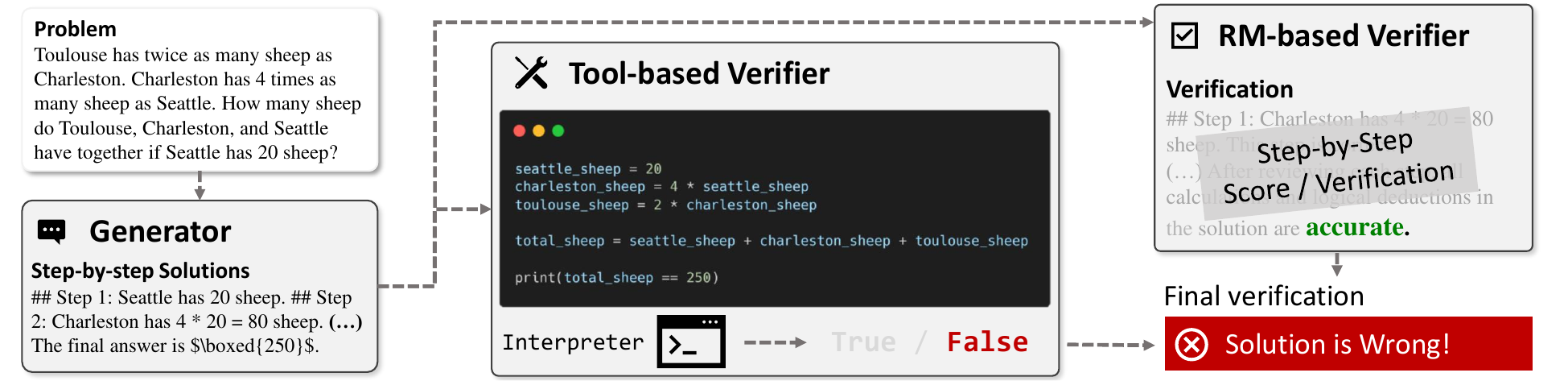}
    \caption{\textbf{Tool-integrated self-verification for mathematical reasoning.}  
(a) \textbf{Generator:} A small language model (sLM) may produce incorrect solutions due to calculation errors.  
(b) \textbf{Tool-based Verifier (ToolV):} The sLM generates executable code based on its reasoning; the output of the code is used to verify the solution’s correctness.  
(c) \textbf{Reward Model (RM)-based Verifier:} The reward model (GenRM / PRM) still evaluates the solution as before, but its verdict only contributes to the final decision if the solution passes the tool-assisted filter.
Concrete examples are in~\autoref{appendix:sec:case_analysis}.}
    \label{fig:method-examples}
\end{figure*}

To address these limitations, we propose a \textbf{Tool-integrated verification (\ours)} approach for parallel test-time scaling in sLMs.
As shown in~\autoref{fig:method-examples}, our verification approach involves two stages: 1) filtering stage with \textbf{Tool-based Verifier (ToolV)} and 2) scoring stage with reward model (RM)-based verifier, and formally can be expressed as:
\begin{equation}
\vy^* = \argmax_{\vy \in \{\vy_1, \ldots, \vy_N\}}  f(\vx, \vy; \mathcal{T}, \theta) \times r(\vx, \vy; \theta), \quad \text{s.t.} \quad \vy_i \sim \pi(\vy \mid \vx, \mI_p; \theta),\  \forall i \in \{1,\ldots, N\}, 
\label{eq:verifier_w_tool}
\end{equation}
where $f(\vx, \vy; \mathcal{T}, \theta) \in \{0, 1\}$ indicate a binary tool-based verification function (with 0 indicating a filtered-out response), $\mathcal{T}$ denotes the utilized tool (e.g., code interpreter, retriever), and $r(\vx, \vy;\theta)$ denotes the verifier score defined in \autoref{eqn:best-of-n}.

\paragraph{Tool-based verifier stage} 
In this stage, sLM utilizes the external tool $\mathcal{T}$, such as a code interpreter or knowledge retriever, to verify generated outputs.
Here, we assume the scenario in which the utilized tool $\mathcal{T}$ is explicitly known.
Specifically, sLM uses these tools to verify numerical accuracy and validate the knowledge in generated solutions.
One specific example is that multiple generated responses are filtered based on tool-based verifiers, discarding those with incorrect calculations or inaccurate knowledge information.
Specifically, the tool-based verification function, $f(\vx, \vy; \mathcal{T},\theta)$, consists of three parts: 1) generating the tool-calling query (i.e., $\vc_1$), 2) the execution of the tool (i.e., $\mathcal{T}(\cdot)$), and 3) extraction of the verification (i.e., $\vc_2$).
Therefore, $f(\vx, \vy; \mathcal{T},\theta)$ can be represented as:
\begin{equation}
f(\vx, \vy; \mathcal{T},\theta) = \vc_2 \sim \pi\left(\vc \mid \mathcal{T}\left(\vc_1\right), \vx, \vy, \mI_f; \theta\right), \quad\text{where}\quad \vc_1 \sim \pi(\vc \mid \vx, \vy, \mI_c; \theta),
\label{eqn:tool-general}
\end{equation}
and $\mI_f$ and $\mI_c$ are task-specific instruction prompts.
Detailed formulation for mathematical reasoning and knowledge-intensive tasks are represented in~\autoref{appendix:detail-method}.

\paragraph{RM-based verifier stage} Following ToolV stage, the remaining generated responses are scored using a reward model, the same model used for generation and filtering. This reward model assesses the overall logical consistency, coherence, and correctness of each response. The final output is chosen as the response with the highest reward score. 

\subsection{Verifier distillation}
\label{sec:distillation}
To further enhance the performance of both verification stages, we employ knowledge distillation~\citep{KD, SeqKD} from LLMs.
Specifically, we fine-tune sLM using tool-based and RM-based verifications generated by a larger teacher model $\theta_T$.
To efficiently manage multiple distinct tasks during the distillation process, we adopt a multi-LoRA~\citep{LoRA} approach, assigning separate LoRA adapters, $\Delta\theta_{\sf tool}$, and $\Delta\theta_{\sf reward}$, for each verifier.

The distillation for ToolV is formulated as:
\begin{equation} 
\mathcal{L}_{\sf tool}(\Delta\theta_{\sf tool}) = - \mathbb{E}_{\vx \sim \mathcal{X}_{\sf train},\ \vy \sim \pi(\cdot\mid \vx, \mI_p; \theta),\ \vc \sim \pi(\cdot\mid \vt, \vx, \vy, \mI; \theta_T)} \log \pi(\vc \mid \vt, \vx, \vy, \mI; \theta+\Delta\theta_{\sf tool}),
\label{eqn:tool-distillation}
\end{equation}
where $\mathcal{X}_{\sf train}$ is training dataset, $\vc \in \{\vc_1, \vc_2\}$, $\mI \in \{\mI_c, \mI_f\}$, and $\vt$ is empty $\phi$ or the output of $\mathcal{T}\left(\vc_1\right)$.
Note that we generate tool-based verifications using the teacher model by applying corresponding instruction $\mI$ in a zero-shot manner~\citep{InstructGPT, FlanT5}.

Similarly, distillation for the RM-based verifier is expressed as: 
\begin{equation} 
\mathcal{L}_{\sf reward}(\Delta\theta_{\sf reward}) = - \mathbb{E}_{\vx \sim \mathcal{X}_{\sf train},\ \vy \sim \pi(\cdot\mid \vx, \mI_p; \theta),\ \vz \sim \pi(\cdot\mid \vx, \vy, \mI_r; \theta_T)} \log \pi(\vz \mid \vx, \vy, \mI_r; \theta+\Delta\theta_{\sf reward}),
\label{eqn:hard-distillation}
\end{equation} 
where $\mI_r$ is RM-based verifier-specific instruction. In~\autoref{eqn:hard-distillation}, responses are first sampled from the student model's proposal distribution. Each sampled response is then verified by the teacher model, and finally, the student model is fine-tuned based on these verifications.

\section{Theoretical analysis}
\label{sec:theoretical-analysis}

In this section, we present theoretical analyses of two aspects with conceptual illustrations: 
(1) how external tools reduce memorization requirements and improve verification performance (\autoref{sec:theoretical-memorization}), and 
(2) how our two-stage approach guarantees improved performance under test-time scaling (\autoref{sec:theoretical-tov-filter}).

\subsection{Memorization bound with \& without tool}\label{sec:theoretical-memorization}
Let us consider a simple verification task in which a verifier assesses whether the given equation $a+b=c$ is true or not.
Let this task's data distribution is $q$ and $P=\left((a,b,c),r\right) \sim q$.
Assume we have $|\mathcal{X}|$ number of training samples such that $\mathcal{X} = \left\{ \left((a,b,c), r\right) \;\middle|\; a,b \in \{0,\dots,M-1\},\; c \in \{0,\dots,2M-2\},\; r=\mathbf{1}_{a+b=c} \right\}$, where $((a,b,c),r)$ is independently sampled according to $q$.
Then, let $X \sim q^{\otimes |\mathcal{X}|}$ be the random variable representing the distribution of the training set $\mathcal{X}$. 
A learning algorithm $\gA$ receives $X$ to produce $\theta$ such that $\theta = \gA\left(X\right)$.
Then, we call that $\gA$ is \emph{\(\varepsilon\)-close-to-optimal} if $\mathrm{err}_{q, |\mathcal{X}|}(\gA) \;\le\; \mathrm{err}_{q, |\mathcal{X}|}\bigl(\gA_{\mathrm{\sf OPT}}\bigr) + \varepsilon$, where $\mathrm{err}_{q,|\mathcal{X}|}(\gA) = \Pr_{X \sim q^{\otimes |\mathcal{X}|}, ((a,b,c),r)\sim q,\hat{r}\sim\pi(\vc|a+b,c, \mI;\theta=\gA(X))}(\hat{r}\neq r)$ and $\gA_{\mathrm{\sf OPT}}$ is the optimal learning algorithm.
Then, \autorefLemma{thm:direct-memorization} shows how much information of $X$ should be memorized within $\theta$ to satisfy almost zero error.

\begin{lemma}[Memorization without Tool~\citep{MemorizationCapacity}]
\label{thm:direct-memorization}
Any learning algorithm $\gA$ that is 
$\varepsilon$-close-to-optimal with sufficiently small 
$\varepsilon>0$ also satisfies $I\left(X;\theta|\;P\right) \;=\; \Omega\!\left(M^{3}\right)$, where $I$ is the mutual information.
\end{lemma}
\begin{proof}[Proof sketch] 
\vspace{-0.1in}
Theorem 1.1 in~\citet{MemorizationCapacity} said that $I(X;\theta\mid P)$ is proportional to at least dataset size. Since $|\mathcal{X}|=2\cdot(M-1)^3$, we can get $\Omega\!\left(M^{3}\right)$. Refer to \autoref{appendix:thm1_proof} for the detailed proof.
\vspace{-0.1in}
\end{proof}

On the other hand, using an external tool that verifies whether $a + b = c$ allows the model to avoid memorizing the full table of sums. Specifically, define a tool $\mathcal{T}$.
Suppose $\theta$ is generated by learning algorithm $\mathcal{A}$ that has access to $\mathcal{T}$, and that $f(a, b, c; \theta, \mathcal{T}) = \mathbf{1}_{\,a + b = c\,}$ holds.
Then we obtain the following result:
\begin{theorem}[Memorization with Tool]
\label{thm:code-snippet} 
Suppose $\theta$ is generated by learning algorithm $\mathcal{A}$ that has access to $\mathcal{T}$, and that $f(a, b, c; \theta, \mathcal{T}) = \mathbf{1}_{\,a + b = c\,}$ holds. Then, any learning algorithm $\gA$ that is 
$\varepsilon$-close-to-optimal with sufficiently small 
$\varepsilon>0$ also satisfies $I\left(X;\;\theta|\;P\right)\;=\;0$, where $I$ is the mutual information.
\end{theorem}
\begin{proof}[Proof sketch]
\vspace{-0.1in}
As the learning algorithm can access an external tool $\mathcal{T}$, then $\theta=\gA(X)$ such that $f(a,b,c;\theta,\mathcal{T})=\mathbf{1}_{\,a+b \;=\;c\,}$. 
$f$ makes $\mathrm{err}_{q,|\mathcal{X}|}(\gA) =0$. Also, $\theta$ is independently determined regardless of $X$, resulting in $I(X;\theta \mid P)=0$.
See \autoref{appendix:thm2_proof} for the detailed proof.
\vspace{-0.1in}
\end{proof}
$I\left(X;\;\theta\mid\;P\right)$ quantifies the amount of information about $X$, drawn from $P$, that must be memorized in $\theta$ learned by $\mathcal{A}$ to achieve near-zero error.
By comparing $I\left(X;\;\theta\mid\;P\right)$ from \autorefLemma{thm:direct-memorization} and \autoref{thm:code-snippet}, we demonstrate that a tool drastically reduces the required memorization of $X$ within $\theta$, lowering $I\left(X;\;\theta\mid\;P\right)$ from $\Omega\!\left(M^{3}\right)$ to $0$.
Consequently, this result implies that with the tool, small models become reliable for the verification task.

\subsection{Effect of tool-based verifier on test-time scaling}\label{sec:theoretical-tov-filter}

We employ the toy setting introduced in \citet{TheoreticalBoN}.
Specifically, for given input $\vx$, the ground-truth label produced by this generator is set as $1$, and the generator $\pi$ produces binary outputs, i.e., $\mathcal{Y}=\{0, 1\}$.
Furthermore, we consider an imperfect verifier inducing noise.
Then, we can show~\autoref{thm:monotonicity} that increasing $q_1$ directly increases the probability of obtaining a correct output from the best-of-$N$.

\begin{theorem}[Monotonicity of Imperfect Verifier]\label{thm:monotonicity}
Let the generator output $0$ or $1$ with equal probability, i.e., $ \pi\left(0|\vx\right) = \pi\left(1|\vx\right) = \frac{1}{2}$, and the verifier $r$ with noise level $p, q$ be defined as follows:
\begin{equation*}
r\left(\vx,0\right) = \begin{cases} 
0, & \text{w.p. } p, \\[1mm]
1, & \text{w.p. } 1-p,
\end{cases}
\quad
r\left(\vx,1\right) = \begin{cases} 
1, & \text{w.p. } q, \\[1mm]
0, & \text{w.p. } 1-q,
\end{cases}
\end{equation*}
with the condition that $p > 1-p$ and $q > 1-q$. 
Assume $\bar{p}$ and $\barbelow{p}$ be the noise level of two verifiers with $\bar{p} > \barbelow{p}$. Then, for any $N \ge 2$, 
\begin{equation}
\pi^N(1\mid \vx)\Big|_{p = \bar{p}} \;>\; \pi^N(1\mid \vx)\Big|_{p = \barbelow{p}}\,.
\end{equation}
\vspace{-0.2in}
\end{theorem}
\begin{proof}[Proof sketch] By the law of total probability, we get $\pi^N(1\mid\vx)$. Then, we can get the monotonicity of $\pi^N(1\mid\vx)$. Refer to \autoref{appendix:thm3_proof} for the detailed proof.
\end{proof}
If tool-based verification function $f$ in~\autoref{sec:tool_verification} effectively acts as a filter for incorrect solutions, we can say that using $f$ increases $p$, thus improving the verifier's capability to choose the correct label as shown in \autoref{thm:monotonicity}.

\section{Experiments}

\subsection{Setup}
\paragraph{Datasets}
We mainly focus on \textbf{mathematical reasoning} task due to its widespread adoption and strong effectiveness in assessing the reasoning capabilities of language models~\citep{ScalingTestTime, 1BSurpass}.
We use (1) \textbf{MATH500}~\citep{MATH, VerifyStepbyStep}, a dataset containing college-level math problems. (2) \textbf{GSM8K}~\citep{GSM8K}, which consists of grade-school math problems. 
We use the training set of each dataset for distillation.
We also include additional experimental results on the subset of MMLU-Pro~\citep{MMLU-Pro} in~\autoref{sec:appendix:additional_results}, which contain a multi-domain knowledge-intensive problems.
\vspace{-0.1in}
\paragraph{Evaluation setting}
Following previous works~\citep{GSM8K, VerifyStepbyStep, ScalingTestTime, 1BSurpass}, we evaluate \textbf{weighted Best-of-N} performance, where we aggregate the score of the solutions ending with the same final answer, to assess test-time compute scaling.
We generate 64 solutions using a fixed generator and measure the percentage of correctly solved problems after verifications.
As a verifier, we use both PRM~\citep{MathShepherd} and GenRM-CoT~\citep{GenRM} (we refer to it as GenRM).
\vspace{-0.1in}
\paragraph{Baselines}
We compare ours, \textbf{Tool-integrated verification}, that utilizes both fine-tuned reward model and tool-based verifiers (\textbf{ToolV}), against the following baselines:
(1) \textbf{Majority Voting}~\citep{self-consistency} (without using verifier),
(2) \textbf{Zero-shot GenRM}~\citep{llmjudge, mindthegap} (without any fine-tuning), 
(3) \textbf{Distilled PRM/GenRM} (with fine-tuning),
(4) \textbf{Themis}~\citep{ToolRM}.
\vspace{-0.1in}
\paragraph{Models \& training}
We experiment with the smallest instruction-tuned models from widely used families: Qwen-2.5-0.5B-Instruct~\citep{Qwen2.5} and Llama-3.2-1B-Instruct~\citep{Llama3}.
In addition, we test SmolLM2-360M-Instruct~\citep{smollm2} for evaluation in extremely small model.
As the teacher model, we employ gpt-4o-mini-2024-07-18~\citep{GPT4o}. 
The teacher model is prompted to generate outputs used to fine-tune student models~\citep{LoRA}. 
For PRM’s teacher, we use Qwen2.5-Math-PRM-7B~\citep{PRMLesson}.

We include more implementation details in~\autoref{appendix:detail-exp}.

\begin{figure*}
    \centering
    \includegraphics[width=\linewidth]{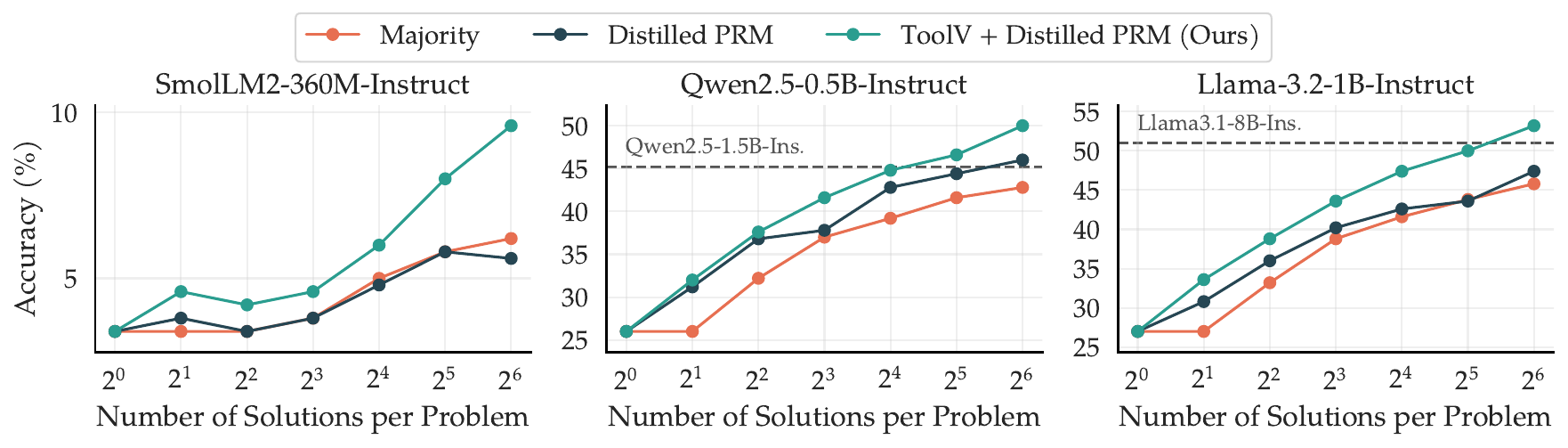}
    \vspace{-0.25in}
    \caption{\textbf{MATH500 with PRM.} Weighted Best-of-N performance of three small language models, emphasizing the benefits of ToolV on college-level math problems. ToolV significantly enhances PRM, enabling small models to outperform or match much larger models. Qwen2.5-1.5B and Llama3.1-8B performances are reported as $N=1$ greedy decoding.} 
    \label{fig:main-prm-math}
    \vspace{-0.15in}
\end{figure*}

\subsection{Experimental results}
\paragraph{ToolV improves PRM in small LMs}
As shown in~\autoref{fig:main-prm-math}, ToolV improves performance when combined with the distilled Process Reward Model (PRM) on the MATH500 benchmark.
Our results show that adding ToolV provides substantial gains in test-time scaling, suggesting that distilled PRM alone is still prone to numerical errors. Notably, with ToolV, \textbf{only using Llama 1B models outperforms the performance of the 8B model}—demonstrating that extra test-time computation can meaningfully boost smaller models, where distilled PRM alone cannot enable the 1B model to reach that performance until generating 64 solutions.
Similarly, ToolV enables Qwen2.5 0.5B to match the performance of the 1.5B model by generating just 16 solutions, showing impressive effectiveness.

\begin{figure*}
    \centering
    \includegraphics[width=\linewidth]{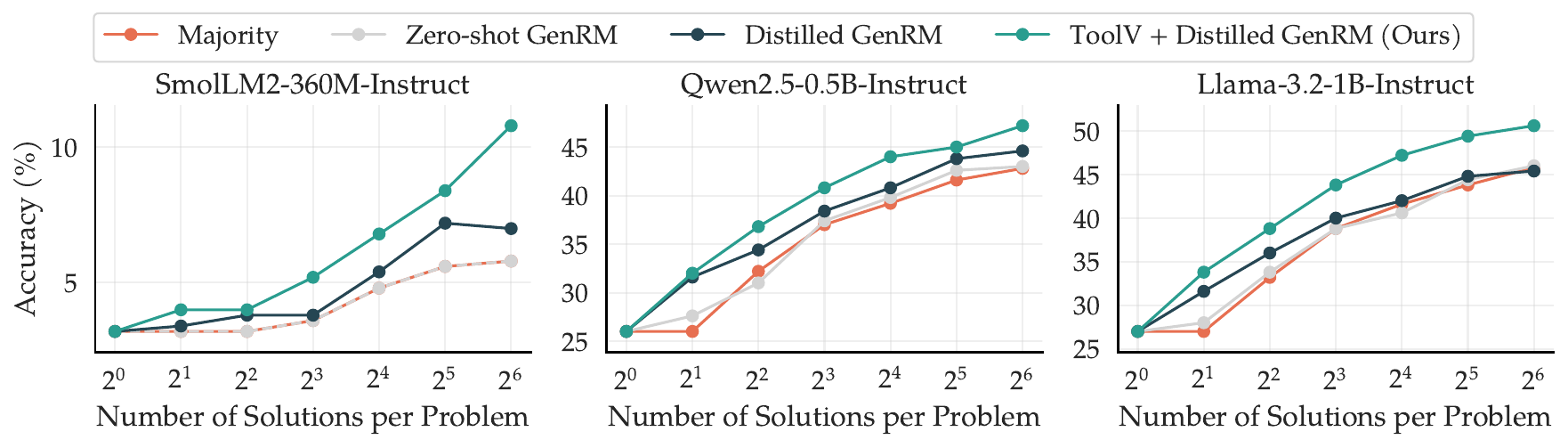}
    \vspace{-0.25in}
    \caption{\textbf{MATH500 with GenRM.} Weighted Best-of-N performance of three small language models, showcasing the effectiveness of ToolV with GenRM, where even generative verification cannot supplement the calculation error which can be easily filtered out by using a tool.} 
    \label{fig:main-math}
    \vspace{-0.15in}
\end{figure*}
\begin{figure*}[t]
    \centering
    \includegraphics[width=\linewidth]{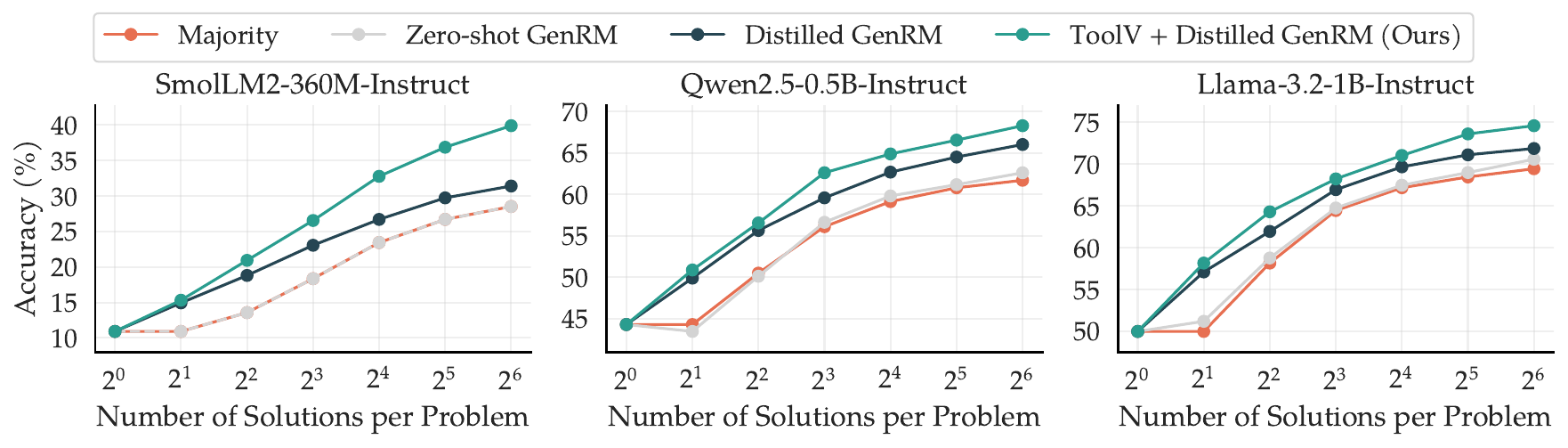}
    \vspace{-0.15in}
    \caption{\textbf{GSM8K with GenRM.} Weighted Best-of-N performance comparison across three small language models. The results show that ToolV also improves model performance on graduate-level arithmetic problems. However, the gains are smaller on this simpler task, where existing verifiers already perform reliably compared to more challenging tasks.} 
    \label{fig:main-gsm8k}
    \vspace{-0.1in}
\end{figure*}

\paragraph{ToolV improves GenRM in sLMs}
As shown in~\autoref{fig:main-math}, ToolV boosts test-time scaling for three small language models on MATH500 when combined with the distilled GenRM~\citep{GenRM}. While GenRM struggles alone, ToolV compensates—at the cost of code generation.
Similar gains appear on GSM8K in~\autoref{fig:main-gsm8k}, especially for SmolLM2-360M-Instruct, the weakest model. This supports our analysis in \autoref{sec:theoretical-memorization} that ToolV enables even small models to memorize key information. Zero-shot GenRM ablations confirm that without distillation, small models struggle to verify solutions~\citep{mindthegap}.

\begin{figure*}[t]
    \centering
    \includegraphics[width=\linewidth]{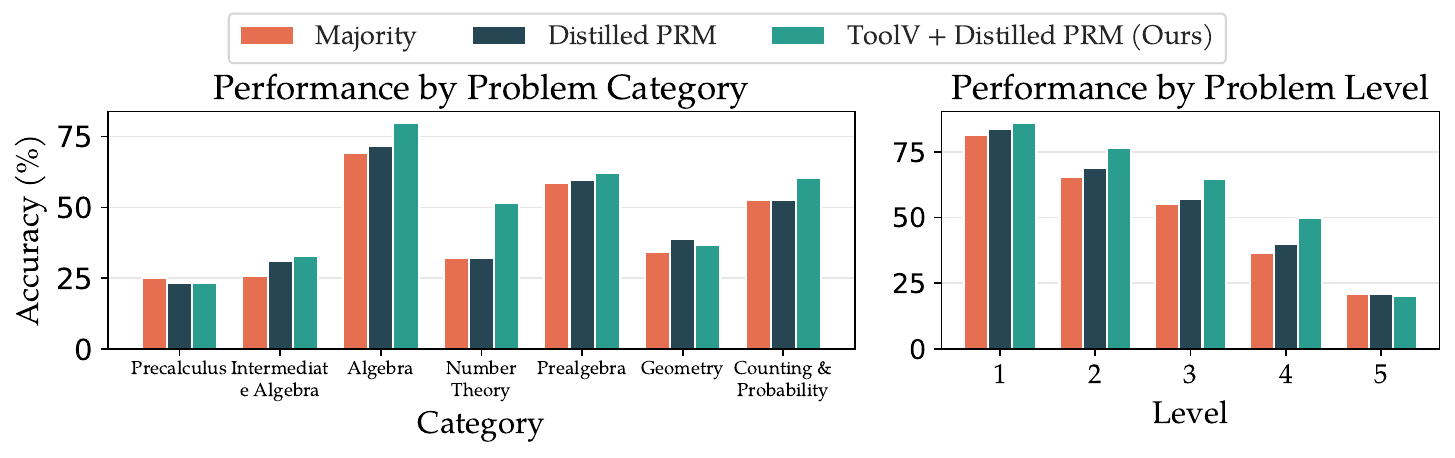}
    \vspace{-0.23in}
    \caption{\textbf{Analysis with problem types and levels.} We perform analysis on the effect of tool-based verifier with problem types and levels in \textbf{MATH500} dataset. The results are from Llama-3.2-1B-Instruct with PRM using weighted Best-of-N ($N=64$). This analysis shows ToolV is most effective on mid-level problems and calculational domains.} 
    \vspace{-0.16in}
    \label{fig:analysis-category}
\end{figure*}

\begin{figure*}[t]
    \centering
    \begin{minipage}[t]{0.645\linewidth}
        \centering
        \includegraphics[width=\linewidth]{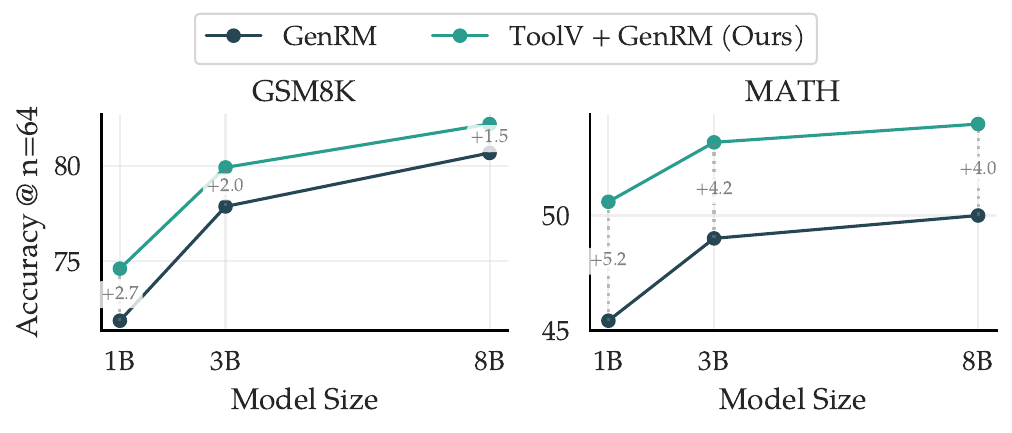}
        \vspace{-0.275in}
        \captionof{figure}{\textbf{Effects of ToolV on sizes of GenRM.} Weighted Best-of-N ($N=64$) performance of GenRM based on different sizes of Llama 3~\citep{Llama3} on \textbf{MATH500}. For ToolV, we use 1B and only scale up the GenRM.}
        \label{fig:analysis-genrm-size}
    \end{minipage}
    \hspace{0.02\linewidth}
    \begin{minipage}[t]{0.315\linewidth}
        \centering
        \includegraphics[width=\linewidth]{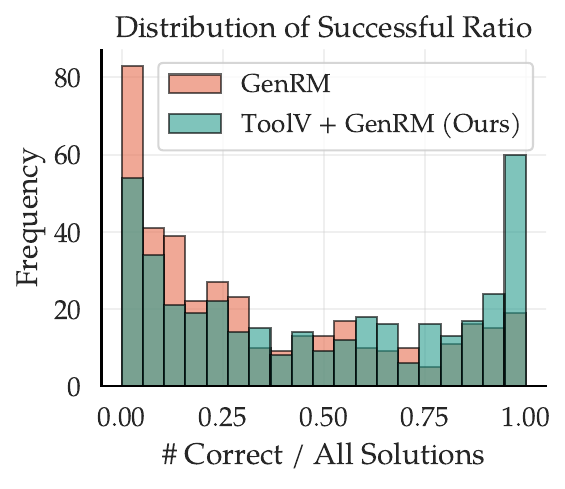}
        \vspace{-0.275in}
        \captionof{figure}{\textbf{Correct solutions ratio} among $N=64$ generations to show how the tool-based verifier works.}
        \label{fig:analysis-correct-ratio}
    \end{minipage}
    \vspace{-0.15in}
\end{figure*}

\paragraph{Our two-stage method outperforms prior tool-integrated verification.}
\begin{wrapfigure}{r}{0.40\textwidth} 
  \centering
  \includegraphics[width=0.38\textwidth]{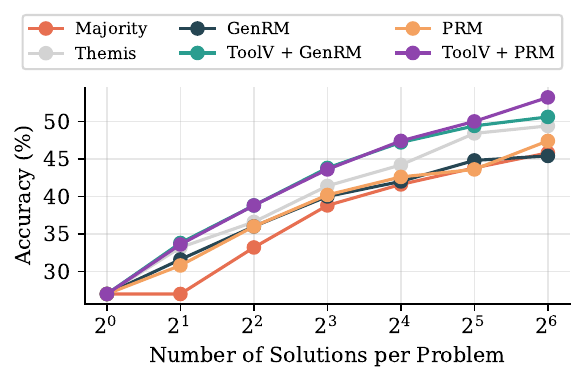}
  \vspace{-0.1in}
  \caption{\textbf{MATH with Themis as baseline.} Weighted Best-of-N performance of Llama-3.2-1B-Instruct model. ToolV outperforms Themis~\citep{ToolRM}.}
  \label{fig:themis}
\end{wrapfigure}
Themis~\citep{ToolRM} demonstrates that integrating external tools can enhance verification performance of 7B models across tasks requiring tools such as calculator, weather, or calendar. 
In contrast, our focus is on mathematical reasoning benchmarks (e.g., GSM8K, MATH500) and much smaller models that were not addressed in~\citet{ToolRM}. 
However, it is notable that our \emph{two-stage} tool-integrated verification approach proves effective for small models on math reasoning, whereas \citet{ToolRM} explored a unified tool-augmented reward modeling framework.

Similar to our method, we first generate tool-integrated verification trajectories using gpt-4o and distill them into the Llama-3.2-1B-Instruct model. 
These trajectories are similar to GenRM but explicitly include intermediate Python code generation~\citep{ToRA}. 
As shown in~\autoref{fig:themis}, Themis~\citep{ToolRM} surpasses other distilled GenRM and PRM baselines without tool usage. 
However, both ToolV + GenRM and ToolV + PRM outperform Themis, indicating that our two-stage approach is better suited for test-time scaling on math reasoning with small models, as it can be combined with distilled PRM and ensures performance improvements as analyzed in~\autoref{sec:theoretical-tov-filter}.

\subsection{Analysis}
\paragraph{Effects of ToolV on category and difficulty}
In~\autoref{fig:analysis-category}, we analyze the $N=64$ weighted Best-of-N performance on MATH500 using Llama-3.2-1B-Instruct.
On the left, category-wise results show ToolV brings clear gains, especially in Algebra, Number Theory, and Counting \& Probability. Geometry sees a drop, likely due to ToolV being less effective in that domain.
On the right, performance by problem level shows consistent improvements with ToolV for Levels 2–4. However, results dip at Level 5, suggesting ToolV struggles with the most challenging problems.
Overall, ToolV works best on mid-level problems and math areas requiring accurate calculation, but improvements are needed for higher-difficulty cases.

\paragraph{ToolV benefits larger verifiers}
\autoref{fig:analysis-genrm-size} shows how performance varies with distilled GenRM size, keeping ToolV fixed at 1B. As GenRM scales, the gap with and without ToolV narrows but remains.
Notably, on MATH500, \textbf{1B GenRM + ToolV outperforms 8B GenRM}, suggesting ToolV can be more effective than scaling the verifier—especially on harder tasks.

\paragraph{Effects of ToolV as filter}
\autoref{fig:analysis-correct-ratio} shows how ToolV acts as an effective filter for mathematical solutions. Using Llama-3.2-1B-Instruct with GenRM on MATH500 ($N=64$ candidates per sample), we recalculated accuracy after applying ToolV to remove wrong outputs. The results support our analysis in \autoref{sec:theoretical-tov-filter}, showing ToolV reliably filters out incorrect solutions and significantly improves accuracy.

\begin{figure*}
    \centering
    \includegraphics[width=\linewidth]{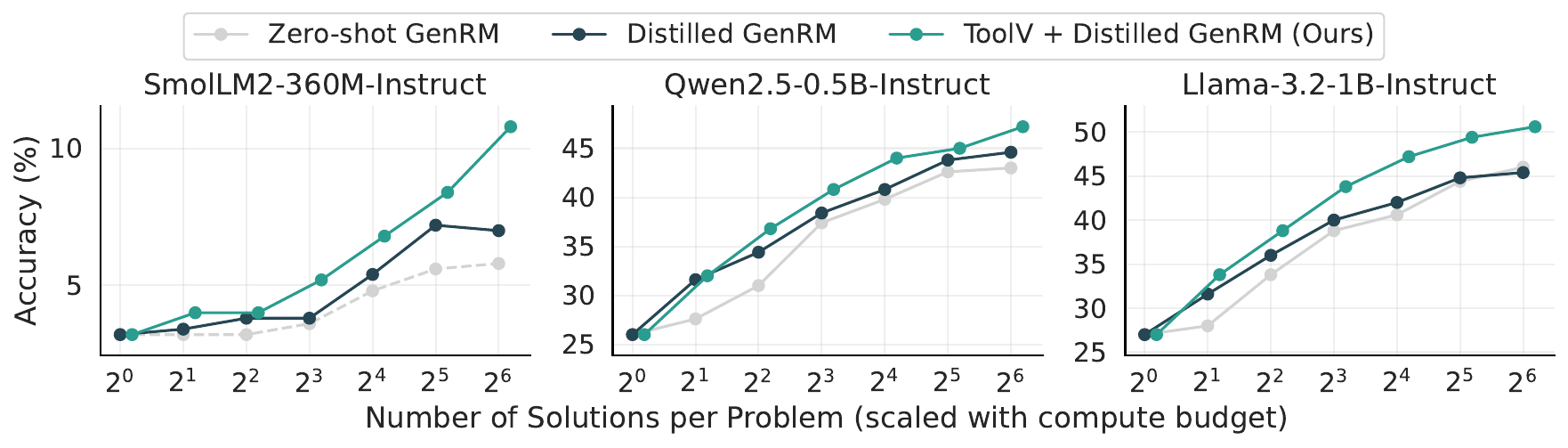}
    \vspace{-0.25in}
    \caption{\textbf{MATH500 with GenRM under a scaled x-axis reflecting compute budget.} Our method remains the best even when compute budget is taken into account.} 
    \label{fig:main-math-scaled}
    \vspace{-0.15in}
\end{figure*}

\begin{wraptable}{r}{0.35\linewidth}
\centering
\vspace{-10pt}
\begin{tabular}{l c}
\toprule
\textbf{Stage} & \textbf{Tokens} \\
\midrule
Solution (Generator) & 574.39 \\
Verification (GenRM)  & 4431.11 \\
Code (ToolV) & 610.84 \\
\bottomrule
\end{tabular}
\vspace{-5pt}
\caption{Token usage per stage.}
\label{tab:token-usage}
\vspace{-10pt}
\end{wraptable}

\subsection{Discussion: Computational overhead of ToolV}
\label{sec:discussion}
ToolV requires generating executable code in addition to producing solutions from the generator and verifier outputs such as GenRM. Understanding the compute budget is therefore important to assess the effectiveness of each method and to determine which approach is preferable under constrained compute resources~\citep{GenRMCompute}. We begin by analyzing the number of generated tokens required by the generator, the verifier (GenRM), and ToolV on the MATH benchmark using Llama-3.2-1B-Instruct. \autoref{tab:token-usage} reports the average token count per solution.

As shown, ToolV introduces extra code generation tokens. This raises two natural questions:
(1) How can we ensure that the performance improvement comes from ToolV itself, rather than from increased compute? If additional budget is available, why not simply use a larger verifier?
(2) Under an equal compute budget, does ToolV still provide benefits compared to using GenRM or PRM alone?
We address these questions below.

\paragraph{Using ToolV is more beneficial than increasing verifier size.}

ToolV enables the effective use of small models, which offers a practical advantage over larger verifiers in terms of GPU memory requirements. Our framework allows small models to function as strong verifiers, which is particularly valuable in memory-constrained environments such as on-device or limited GPU setups.

Even when accounting for the extra compute, the cost of ToolV corresponds to using only a 1.14 times larger verifier. Using~\autoref{tab:token-usage}, let $k$ denote the relative scale factor. Then, $k = (5616.34 - 574.39) / 4431.11 = 1.14$. 
Since ToolV 1B + GenRM 1B surpasses GenRM 8B on MATH (\autoref{fig:analysis-genrm-size}), the performance gain from ToolV more than justifies this small additional cost, especially when compared with simply scaling up the verifier.

\paragraph{ToolV still improves performance under the same compute budget.}

To account for the overhead of ToolV, we shift the x axis in performance plots such as~\autoref{fig:main-math}. Applying the scaling factor of $1.14$, we shift the x-axis of our method with GenRM accordingly. 
The resulting comparison in~\autoref{fig:main-math-scaled} shows that even after budget normalization, ToolV continues to provide meaningful gains, particularly for the smallest model, SmolLM2-360M-Instruct, and in settings with the large number of generated solutions.

Additional discussion and extended experimental results are provided in~\autoref{sec:appendix:additional_results}.
\section{Conclusion}
In this work, we introduced \textbf{Tool-integrated Verification (T1)}, which delegates memorization-intensive tasks in verification to external tools for sLMs.
Our method involves a tool-based verification stage and a reward-model-based scoring stage, both enhanced by knowledge distillation from large verifiers.
Theoretical analysis confirmed that tool use substantially reduces the memorization burden on sLMs and improves test-time scaling accuracy.
Empirical experiments demonstrated that T1 significantly improves the test-time scaling performance of sLM on mathematical reasoning and knowledge-intensive tasks.
A key conclusion of our work is that tool integration is essential for enhancing sLM performance, even under test-time scaling, by reducing the memorization burden.
\paragraph{Limitations \& Future Works}
While T1 shows strong improvements, some limitations remain.
(1) ToolV acts only as a rejection filter and cannot recover from false negatives—correct solutions mistakenly rejected by the verifier. As one possible implementation of T1, this limitation could be mitigated by integrating tool-based reasoning into the verification step, allowing the verifier to leverage correctness guarantees from tool outputs~\citep{ToRA, ToolRM}, which we do not explore in this work.
(2) Our work focuses on best-of-N (parallel) test-time scaling, which lacks information sharing between generations. However, tools can also benefit other test-time scaling strategies, such as step-level search~\citep{ToT} in sLMs or long reasoning chains in sequential test-time scaling as demonstrated by~\citet{START}. Exploring these directions presents a promising avenue for future work.

\section*{Ethics statement}
This work uses only public datasets for math and knowledge tasks and does not involve human subjects or personal data. Tool based verification runs a code interpreter for numeric checks and a retriever over Wikipedia abstracts for factual checks, and these tools do not store user data. The main risks are retrieval errors and code execution failures; we bound both with conservative rules, simple evaluation code, and we discard unverifiable claims. We follow all dataset and model licenses, disclose model families, teacher models, and training compute, and report the resources used.

\section*{Reproducibility statement}
We specify datasets, baselines, metrics, and the evaluation protocol for weighted Best of N with both a process reward model and a generative verifier. The appendix provides implementation details, training setup, key hyperparameters, and the exact prompts for code based math checks and document based fact checks. We state the retriever, source, and document count per query, and we define the rules for code execution and success signals. We will release an anonymous package with scripts, prompts, configuration. Plus, we will provide instructions and seeds to recreate all tables and figures.

\bibliography{iclr2026_conference}

@article{ScalingTestTime,
  author       = {Charlie Snell and
                  Jaehoon Lee and
                  Kelvin Xu and
                  Aviral Kumar},
  title        = {Scaling {LLM} Test-Time Compute Optimally can be More Effective than
                  Scaling Model Parameters},
  journal      = {arXiv},
  volume       = {2408.03314},
  year         = {2024},
  url          = {https://doi.org/10.48550/arXiv.2408.03314},
  eprinttype    = {arXiv},
  eprint       = {2408.03314},
}

@inproceedings{VerifyStepbyStep,
  author       = {Hunter Lightman and
                  Vineet Kosaraju and
                  Yuri Burda and
                  Harrison Edwards and
                  Bowen Baker and
                  Teddy Lee and
                  Jan Leike and
                  John Schulman and
                  Ilya Sutskever and
                  Karl Cobbe},
  title        = {Let's Verify Step by Step},
  booktitle    = {The Twelfth International Conference on Learning Representations,
                  {ICLR} 2024, Vienna, Austria, May 7-11, 2024},
  publisher    = {OpenReview.net},
  year         = {2024},
  url          = {https://openreview.net/forum?id=v8L0pN6EOi},
}

@article{GSM8K,
  author       = {Karl Cobbe and
                  Vineet Kosaraju and
                  Mohammad Bavarian and
                  Mark Chen and
                  Heewoo Jun and
                  Lukasz Kaiser and
                  Matthias Plappert and
                  Jerry Tworek and
                  Jacob Hilton and
                  Reiichiro Nakano and
                  Christopher Hesse and
                  John Schulman},
  title        = {Training Verifiers to Solve Math Word Problems},
  journal      = {arXiv},
  volume       = {2110.14168},
  year         = {2021},
  url          = {https://arxiv.org/abs/2110.14168},
  eprinttype    = {arXiv},
  eprint       = {2110.14168},
}

@article{LargeLanguageMonkeys,
  author       = {Bradley C. A. Brown and
                  Jordan Juravsky and
                  Ryan Saul Ehrlich and
                  Ronald Clark and
                  Quoc V. Le and
                  Christopher R{\'{e}} and
                  Azalia Mirhoseini},
  title        = {Large Language Monkeys: Scaling Inference Compute with Repeated Sampling},
  journal      = {arXiv},
  volume       = {2407.21787},
  year         = {2024},
  url          = {https://doi.org/10.48550/arXiv.2407.21787},
  eprinttype    = {arXiv},
  eprint       = {2407.21787},
}

@inproceedings{ToT,
  author       = {Shunyu Yao and
                  Dian Yu and
                  Jeffrey Zhao and
                  Izhak Shafran and
                  Tom Griffiths and
                  Yuan Cao and
                  Karthik Narasimhan},
  editor       = {Alice Oh and
                  Tristan Naumann and
                  Amir Globerson and
                  Kate Saenko and
                  Moritz Hardt and
                  Sergey Levine},
  title        = {Tree of Thoughts: Deliberate Problem Solving with Large Language Models},
  booktitle    = {Advances in Neural Information Processing Systems 36: Annual Conference
                  on Neural Information Processing Systems 2023, NeurIPS 2023, New Orleans,
                  LA, USA, December 10 - 16, 2023},
  year         = {2023},
  url          = {http://papers.nips.cc/paper\_files/paper/2023/hash/271db9922b8d1f4dd7aaef84ed5ac703 \\
                  -Abstract-Conference.html}
}

@inproceedings{MathShepherd,
  author       = {Peiyi Wang and
                  Lei Li and
                  Zhihong Shao and
                  Runxin Xu and
                  Damai Dai and
                  Yifei Li and
                  Deli Chen and
                  Yu Wu and
                  Zhifang Sui},
  editor       = {Lun{-}Wei Ku and
                  Andre Martins and
                  Vivek Srikumar},
  title        = {Math-Shepherd: Verify and Reinforce LLMs Step-by-step without Human
                  Annotations},
  booktitle    = {Proceedings of the 62nd Annual Meeting of the Association for Computational
                  Linguistics (Volume 1: Long Papers), {ACL} 2024, Bangkok, Thailand,
                  August 11-16, 2024},
  pages        = {9426--9439},
  publisher    = {Association for Computational Linguistics},
  year         = {2024},
  url          = {https://doi.org/10.18653/v1/2024.acl-long.510},
}

@article{GenRM,
  author       = {Lunjun Zhang and
                  Arian Hosseini and
                  Hritik Bansal and
                  Mehran Kazemi and
                  Aviral Kumar and
                  Rishabh Agarwal},
  title        = {Generative Verifiers: Reward Modeling as Next-Token Prediction},
  journal      = {arXiv},
  volume       = {2408.15240},
  year         = {2024},
  url          = {https://doi.org/10.48550/arXiv.2408.15240},
  eprinttype    = {arXiv},
  eprint       = {2408.15240},
}

@inproceedings{CoT,
  title={Chain-of-Thought Prompting Elicits Reasoning in Large Language Models},
  author={Wei, Jason and Wang, Xuezhi and Schuurmans, Dale and Bosma, Maarten and Xia, Fei and Chi, Ed H and Le, Quoc V and Zhou, Denny and others},
  booktitle={Advances in Neural Information Processing Systems},
  year={2022},
}

@inproceedings{MATH,
  author       = {Dan Hendrycks and
                  Collin Burns and
                  Saurav Kadavath and
                  Akul Arora and
                  Steven Basart and
                  Eric Tang and
                  Dawn Song and
                  Jacob Steinhardt},
  editor       = {Joaquin Vanschoren and
                  Sai{-}Kit Yeung},
  title        = {Measuring Mathematical Problem Solving With the {MATH} Dataset},
  booktitle    = {Proceedings of the Neural Information Processing Systems Track on
                  Datasets and Benchmarks 1, NeurIPS Datasets and Benchmarks 2021, December
                  2021, virtual},
  year         = {2021},
  url          = {https://datasets-benchmarks-proceedings.neurips.cc/paper/2021/hash/be83ab3ecd0db773eb2dc1b0a17836a1-Abstract-round2.html},
}

@article{SelfReflection,
  author       = {Aviral Kumar and
                  Vincent Zhuang and
                  Rishabh Agarwal and
                  Yi Su and
                  John D. Co{-}Reyes and
                  Avi Singh and
                  Kate Baumli and
                  Shariq Iqbal and
                  Colton Bishop and
                  Rebecca Roelofs and
                  Lei M. Zhang and
                  Kay McKinney and
                  Disha Shrivastava and
                  Cosmin Paduraru and
                  George Tucker and
                  Doina Precup and
                  Feryal M. P. Behbahani and
                  Aleksandra Faust},
  title        = {Training Language Models to Self-Correct via Reinforcement Learning},
  journal      = {CoRR},
  volume       = {abs/2409.12917},
  year         = {2024},
  url          = {https://doi.org/10.48550/arXiv.2409.12917}
}

@inproceedings{llmjudge,
  author       = {Lianmin Zheng and
                  Wei{-}Lin Chiang and
                  Ying Sheng and
                  Siyuan Zhuang and
                  Zhanghao Wu and
                  Yonghao Zhuang and
                  Zi Lin and
                  Zhuohan Li and
                  Dacheng Li and
                  Eric P. Xing and
                  Hao Zhang and
                  Joseph E. Gonzalez and
                  Ion Stoica},
  editor       = {Alice Oh and
                  Tristan Naumann and
                  Amir Globerson and
                  Kate Saenko and
                  Moritz Hardt and
                  Sergey Levine},
  title        = {Judging LLM-as-a-Judge with MT-Bench and Chatbot Arena},
  booktitle    = {Advances in Neural Information Processing Systems 36: Annual Conference
                  on Neural Information Processing Systems 2023, NeurIPS 2023, New Orleans,
                  LA, USA, December 10 - 16, 2023},
  year         = {2023},
  url          = {http://papers.nips.cc/paper\_files/paper/2023/hash/91f18a1287b398d378ef22505bf41832-Abstract-Datasets\_and\_Benchmarks.html}
}

@article{KD,
  author       = {Geoffrey E. Hinton and
                  Oriol Vinyals and
                  Jeffrey Dean},
  title        = {Distilling the Knowledge in a Neural Network},
  journal      = {arXiv},
  volume       = {1503.02531},
  year         = {2015},
  url          = {http://arxiv.org/abs/1503.02531},
  eprinttype    = {arXiv},
  eprint       = {1503.02531},
}

@article{mindthegap,
  author       = {Yuda Song and
                  Hanlin Zhang and
                  Carson Eisenach and
                  Sham M. Kakade and
                  Dean P. Foster and
                  Udaya Ghai},
  title        = {Mind the Gap: Examining the Self-Improvement Capabilities of Large
                  Language Models},
  journal      = {arXiv},
  volume       = {2412.02674},
  year         = {2024},
  url          = {https://doi.org/10.48550/arXiv.2412.02674},
  eprinttype    = {arXiv},
  eprint       = {2412.02674},
}

@inproceedings{SeqKD,
  author       = {Yoon Kim and
                  Alexander M. Rush},
  editor       = {Jian Su and
                  Xavier Carreras and
                  Kevin Duh},
  title        = {Sequence-Level Knowledge Distillation},
  booktitle    = {Proceedings of the 2016 Conference on Empirical Methods in Natural
                  Language Processing, {EMNLP} 2016, Austin, Texas, USA, November 1-4,
                  2016},
  pages        = {1317--1327},
  publisher    = {The Association for Computational Linguistics},
  year         = {2016},
  url          = {https://doi.org/10.18653/v1/d16-1139},
}

@inproceedings{GPT3,
  author       = {Tom B. Brown and
                  Benjamin Mann and
                  Nick Ryder and
                  Melanie Subbiah and
                  Jared Kaplan and
                  Prafulla Dhariwal and
                  Arvind Neelakantan and
                  Pranav Shyam and
                  Girish Sastry and
                  Amanda Askell and
                  Sandhini Agarwal and
                  Ariel Herbert{-}Voss and
                  Gretchen Krueger and
                  Tom Henighan and
                  Rewon Child and
                  Aditya Ramesh and
                  Daniel M. Ziegler and
                  Jeffrey Wu and
                  Clemens Winter and
                  Christopher Hesse and
                  Mark Chen and
                  Eric Sigler and
                  Mateusz Litwin and
                  Scott Gray and
                  Benjamin Chess and
                  Jack Clark and
                  Christopher Berner and
                  Sam McCandlish and
                  Alec Radford and
                  Ilya Sutskever and
                  Dario Amodei},
  editor       = {Hugo Larochelle and
                  Marc'Aurelio Ranzato and
                  Raia Hadsell and
                  Maria{-}Florina Balcan and
                  Hsuan{-}Tien Lin},
  title        = {Language Models are Few-Shot Learners},
  booktitle    = {Advances in Neural Information Processing Systems 33: Annual Conference
                  on Neural Information Processing Systems 2020, NeurIPS 2020, December
                  6-12, 2020, virtual},
  year         = {2020},
  url          = {https://proceedings.neurips.cc/paper/2020/hash/1457c0d6bfcb4967418bfb8ac142f64a-Abstract.html}
}

@article{Gemini,
  author       = {Machel Reid and
                  Nikolay Savinov and
                  Denis Teplyashin and
                  Dmitry Lepikhin and
                  Timothy P. Lillicrap and
                  Jean{-}Baptiste Alayrac and
                  Radu Soricut and
                  Angeliki Lazaridou and
                  Orhan Firat and
                  Julian Schrittwieser and
                  Ioannis Antonoglou and
                  Rohan Anil and
                  Sebastian Borgeaud and
                  Andrew M. Dai and
                  Katie Millican and
                  Ethan Dyer and
                  Mia Glaese and
                  Thibault Sottiaux and
                  Benjamin Lee and
                  Fabio Viola and
                  Malcolm Reynolds and
                  Yuanzhong Xu and
                  James Molloy and
                  Jilin Chen and
                  Michael Isard and
                  Paul Barham and
                  Tom Hennigan and
                  Ross McIlroy and
                  Melvin Johnson and
                  Johan Schalkwyk and
                  Eli Collins and
                  Eliza Rutherford and
                  Erica Moreira and
                  Kareem Ayoub and
                  Megha Goel and
                  Clemens Meyer and
                  Gregory Thornton and
                  Zhen Yang and
                  Henryk Michalewski and
                  Zaheer Abbas and
                  Nathan Schucher and
                  Ankesh Anand and
                  Richard Ives and
                  James Keeling and
                  Karel Lenc and
                  Salem Haykal and
                  Siamak Shakeri and
                  Pranav Shyam and
                  Aakanksha Chowdhery and
                  Roman Ring and
                  Stephen Spencer and
                  Eren Sezener and
                  et al.},
  title        = {Gemini 1.5: Unlocking multimodal understanding across millions of
                  tokens of context},
  journal      = {arXiv},
  volume       = {2403.05530},
  year         = {2024},
  url          = {https://doi.org/10.48550/arXiv.2403.05530},
}

@article{GPT4o,
  author       = {Aaron Hurst and
                  Adam Lerer and
                  Adam P. Goucher and
                  Adam Perelman and
                  Aditya Ramesh and
                  Aidan Clark and
                  AJ Ostrow and
                  Akila Welihinda and
                  Alan Hayes and
                  Alec Radford and
                  Aleksander Madry and
                  Alex Baker{-}Whitcomb and
                  Alex Beutel and
                  Alex Borzunov and
                  Alex Carney and
                  Alex Chow and
                  Alex Kirillov and
                  Alex Nichol and
                  Alex Paino and
                  Alex Renzin and
                  Alex Tachard Passos and
                  Alexander Kirillov and
                  Alexi Christakis and
                  Alexis Conneau and
                  Ali Kamali and
                  Allan Jabri and
                  Allison Moyer and
                  Allison Tam and
                  Amadou Crookes and
                  Amin Tootoonchian and
                  Ananya Kumar and
                  Andrea Vallone and
                  Andrej Karpathy and
                  Andrew Braunstein and
                  Andrew Cann and
                  Andrew Codispoti and
                  Andrew Galu and
                  Andrew Kondrich and
                  Andrew Tulloch and
                  Andrey Mishchenko and
                  Angela Baek and
                  Angela Jiang and
                  Antoine Pelisse and
                  Antonia Woodford and
                  Anuj Gosalia and
                  Arka Dhar and
                  Ashley Pantuliano and
                  Avi Nayak and
                  Avital Oliver and
                  Barret Zoph and
                  Behrooz Ghorbani and
                  Ben Leimberger and
                  Ben Rossen and
                  Ben Sokolowsky and
                  Ben Wang and
                  Benjamin Zweig and
                  Beth Hoover and
                  Blake Samic and
                  Bob McGrew and
                  Bobby Spero and
                  Bogo Giertler and
                  Bowen Cheng and
                  Brad Lightcap and
                  Brandon Walkin and
                  Brendan Quinn and
                  Brian Guarraci and
                  Brian Hsu and
                  Bright Kellogg and
                  Brydon Eastman and
                  Camillo Lugaresi and
                  Carroll L. Wainwright and
                  Cary Bassin and
                  Cary Hudson and
                  Casey Chu and
                  Chad Nelson and
                  Chak Li and
                  Chan Jun Shern and
                  Channing Conger and
                  Charlotte Barette and
                  Chelsea Voss and
                  Chen Ding and
                  Cheng Lu and
                  Chong Zhang and
                  Chris Beaumont and
                  Chris Hallacy and
                  Chris Koch and
                  Christian Gibson and
                  Christina Kim and
                  Christine Choi and
                  Christine McLeavey and
                  Christopher Hesse and
                  Claudia Fischer and
                  Clemens Winter and
                  Coley Czarnecki and
                  Colin Jarvis and
                  Colin Wei and
                  Constantin Koumouzelis and
                  Dane Sherburn},
  title        = {GPT-4o System Card},
  journal      = {arXiv},
  volume       = {2410.21276},
  year         = {2024},
  url          = {https://doi.org/10.48550/arXiv.2410.21276}
}

@misc{R1,
      title={DeepSeek-R1: Incentivizing Reasoning Capability in LLMs via Reinforcement Learning}, 
      author={DeepSeek-AI and Daya Guo and Dejian Yang and Haowei Zhang and Junxiao Song and Ruoyu Zhang and Runxin Xu and Qihao Zhu and Shirong Ma and Peiyi Wang and Xiao Bi and Xiaokang Zhang and Xingkai Yu and Yu Wu and Z. F. Wu and Zhibin Gou and Zhihong Shao and Zhuoshu Li and Ziyi Gao and Aixin Liu and Bing Xue and Bingxuan Wang and Bochao Wu and Bei Feng and Chengda Lu and Chenggang Zhao and Chengqi Deng and Chenyu Zhang and Chong Ruan and Damai Dai and Deli Chen and Dongjie Ji and Erhang Li and Fangyun Lin and Fucong Dai and Fuli Luo and Guangbo Hao and Guanting Chen and Guowei Li and H. Zhang and Han Bao and Hanwei Xu and Haocheng Wang and Honghui Ding and Huajian Xin and Huazuo Gao and Hui Qu and Hui Li and Jianzhong Guo and Jiashi Li and Jiawei Wang and Jingchang Chen and Jingyang Yuan and Junjie Qiu and Junlong Li and J. L. Cai and Jiaqi Ni and Jian Liang and Jin Chen and Kai Dong and Kai Hu and Kaige Gao and Kang Guan and Kexin Huang and Kuai Yu and Lean Wang and Lecong Zhang and Liang Zhao and Litong Wang and Liyue Zhang and Lei Xu and Leyi Xia and Mingchuan Zhang and Minghua Zhang and Minghui Tang and Meng Li and Miaojun Wang and Mingming Li and Ning Tian and Panpan Huang and Peng Zhang and Qiancheng Wang and Qinyu Chen and Qiushi Du and Ruiqi Ge and Ruisong Zhang and Ruizhe Pan and Runji Wang and R. J. Chen and R. L. Jin and Ruyi Chen and Shanghao Lu and Shangyan Zhou and Shanhuang Chen and Shengfeng Ye and Shiyu Wang and Shuiping Yu and Shunfeng Zhou and Shuting Pan and S. S. Li and Shuang Zhou and Shaoqing Wu and Shengfeng Ye and Tao Yun and Tian Pei and Tianyu Sun and T. Wang and Wangding Zeng and Wanjia Zhao and Wen Liu and Wenfeng Liang and Wenjun Gao and Wenqin Yu and Wentao Zhang and W. L. Xiao and Wei An and Xiaodong Liu and Xiaohan Wang and Xiaokang Chen and Xiaotao Nie and Xin Cheng and Xin Liu and Xin Xie and Xingchao Liu and Xinyu Yang and Xinyuan Li and Xuecheng Su and Xuheng Lin and X. Q. Li and Xiangyue Jin and Xiaojin Shen and Xiaosha Chen and Xiaowen Sun and Xiaoxiang Wang and Xinnan Song and Xinyi Zhou and Xianzu Wang and Xinxia Shan and Y. K. Li and Y. Q. Wang and Y. X. Wei and Yang Zhang and Yanhong Xu and Yao Li and Yao Zhao and Yaofeng Sun and Yaohui Wang and Yi Yu and Yichao Zhang and Yifan Shi and Yiliang Xiong and Ying He and Yishi Piao and Yisong Wang and Yixuan Tan and Yiyang Ma and Yiyuan Liu and Yongqiang Guo and Yuan Ou and Yuduan Wang and Yue Gong and Yuheng Zou and Yujia He and Yunfan Xiong and Yuxiang Luo and Yuxiang You and Yuxuan Liu and Yuyang Zhou and Y. X. Zhu and Yanhong Xu and Yanping Huang and Yaohui Li and Yi Zheng and Yuchen Zhu and Yunxian Ma and Ying Tang and Yukun Zha and Yuting Yan and Z. Z. Ren and Zehui Ren and Zhangli Sha and Zhe Fu and Zhean Xu and Zhenda Xie and Zhengyan Zhang and Zhewen Hao and Zhicheng Ma and Zhigang Yan and Zhiyu Wu and Zihui Gu and Zijia Zhu and Zijun Liu and Zilin Li and Ziwei Xie and Ziyang Song and Zizheng Pan and Zhen Huang and Zhipeng Xu and Zhongyu Zhang and Zhen Zhang},
      year={2025},
      eprint={2501.12948},
      archivePrefix={arXiv},
      primaryClass={cs.CL},
      url={https://arxiv.org/abs/2501.12948}, 
}

@misc{s1,
      title={s1: Simple test-time scaling}, 
      author={Niklas Muennighoff and Zitong Yang and Weijia Shi and Xiang Lisa Li and Li Fei-Fei and Hannaneh Hajishirzi and Luke Zettlemoyer and Percy Liang and Emmanuel Candès and Tatsunori Hashimoto},
      year={2025},
      eprint={2501.19393},
      archivePrefix={arXiv},
      primaryClass={cs.CL},
      url={https://arxiv.org/abs/2501.19393}, 
}

@inproceedings{KARD,
  author       = {Minki Kang and
                  Seanie Lee and
                  Jinheon Baek and
                  Kenji Kawaguchi and
                  Sung Ju Hwang},
  editor       = {Alice Oh and
                  Tristan Naumann and
                  Amir Globerson and
                  Kate Saenko and
                  Moritz Hardt and
                  Sergey Levine},
  title        = {Knowledge-Augmented Reasoning Distillation for Small Language Models
                  in Knowledge-Intensive Tasks},
  booktitle    = {Advances in Neural Information Processing Systems 36: Annual Conference
                  on Neural Information Processing Systems 2023, NeurIPS 2023, New Orleans,
                  LA, USA, December 10 - 16, 2023},
  year         = {2023},
  url          = {http://papers.nips.cc/paper\_files/paper/2023/hash/97faedc90260eae5c400f92d5831c3d7-Abstract-Conference.html},
}

@article{ProcessBench,
  author       = {Chujie Zheng and
                  Zhenru Zhang and
                  Beichen Zhang and
                  Runji Lin and
                  Keming Lu and
                  Bowen Yu and
                  Dayiheng Liu and
                  Jingren Zhou and
                  Junyang Lin},
  title        = {ProcessBench: Identifying Process Errors in Mathematical Reasoning},
  journal      = {arXiv},
  volume       = {2412.06559},
  year         = {2024},
  url          = {https://doi.org/10.48550/arXiv.2412.06559}
}

@misc{PRMLesson,
      title={The Lessons of Developing Process Reward Models in Mathematical Reasoning}, 
      author={Zhenru Zhang and Chujie Zheng and Yangzhen Wu and Beichen Zhang and Runji Lin and Bowen Yu and Dayiheng Liu and Jingren Zhou and Junyang Lin},
      year={2025},
      eprint={2501.07301},
      archivePrefix={arXiv},
      primaryClass={cs.CL},
      url={https://arxiv.org/abs/2501.07301}, 
}

@misc{1BSurpass,
      title={Can 1B LLM Surpass 405B LLM? Rethinking Compute-Optimal Test-Time Scaling}, 
      author={Runze Liu and Junqi Gao and Jian Zhao and Kaiyan Zhang and Xiu Li and Biqing Qi and Wanli Ouyang and Bowen Zhou},
      year={2025},
      eprint={2502.06703},
      archivePrefix={arXiv},
      primaryClass={cs.CL},
      url={https://arxiv.org/abs/2502.06703}, 
}

@inproceedings{ToRA,
  author       = {Zhibin Gou and
                  Zhihong Shao and
                  Yeyun Gong and
                  Yelong Shen and
                  Yujiu Yang and
                  Minlie Huang and
                  Nan Duan and
                  Weizhu Chen},
  title        = {ToRA: {A} Tool-Integrated Reasoning Agent for Mathematical Problem
                  Solving},
  booktitle    = {The Twelfth International Conference on Learning Representations,
                  {ICLR} 2024, Vienna, Austria, May 7-11, 2024},
  publisher    = {OpenReview.net},
  year         = {2024},
  url          = {https://openreview.net/forum?id=Ep0TtjVoap},
}

@inproceedings{Toolformer,
  author       = {Timo Schick and
                  Jane Dwivedi{-}Yu and
                  Roberto Dess{\`{\i}} and
                  Roberta Raileanu and
                  Maria Lomeli and
                  Eric Hambro and
                  Luke Zettlemoyer and
                  Nicola Cancedda and
                  Thomas Scialom},
  editor       = {Alice Oh and
                  Tristan Naumann and
                  Amir Globerson and
                  Kate Saenko and
                  Moritz Hardt and
                  Sergey Levine},
  title        = {Toolformer: Language Models Can Teach Themselves to Use Tools},
  booktitle    = {Advances in Neural Information Processing Systems 36: Annual Conference
                  on Neural Information Processing Systems 2023, NeurIPS 2023, New Orleans,
                  LA, USA, December 10 - 16, 2023},
  year         = {2023},
  url          = {http://papers.nips.cc/paper\_files/paper/2023/hash/d842425e4bf79ba039352da0f658a906-Abstract-Conference.html},
}

@inproceedings{self-consistency,
  author       = {Xuezhi Wang and
                  Jason Wei and
                  Dale Schuurmans and
                  Quoc V. Le and
                  Ed H. Chi and
                  Sharan Narang and
                  Aakanksha Chowdhery and
                  Denny Zhou},
  title        = {Self-Consistency Improves Chain of Thought Reasoning in Language Models},
  booktitle    = {The Eleventh International Conference on Learning Representations,
                  {ICLR} 2023, Kigali, Rwanda, May 1-5, 2023},
  publisher    = {OpenReview.net},
  year         = {2023},
  url          = {https://openreview.net/forum?id=1PL1NIMMrw},
}

@article{Qwen2.5,
  author       = {An Yang and
                  Baosong Yang and
                  Beichen Zhang and
                  Binyuan Hui and
                  Bo Zheng and
                  Bowen Yu and
                  Chengyuan Li and
                  Dayiheng Liu and
                  Fei Huang and
                  Haoran Wei and
                  Huan Lin and
                  Jian Yang and
                  Jianhong Tu and
                  Jianwei Zhang and
                  Jianxin Yang and
                  Jiaxi Yang and
                  Jingren Zhou and
                  Junyang Lin and
                  Kai Dang and
                  Keming Lu and
                  Keqin Bao and
                  Kexin Yang and
                  Le Yu and
                  Mei Li and
                  Mingfeng Xue and
                  Pei Zhang and
                  Qin Zhu and
                  Rui Men and
                  Runji Lin and
                  Tianhao Li and
                  Tingyu Xia and
                  Xingzhang Ren and
                  Xuancheng Ren and
                  Yang Fan and
                  Yang Su and
                  Yichang Zhang and
                  Yu Wan and
                  Yuqiong Liu and
                  Zeyu Cui and
                  Zhenru Zhang and
                  Zihan Qiu},
  title        = {Qwen2.5 Technical Report},
  journal      = {arXiv},
  volume       = {2412.15115},
  year         = {2024},
  url          = {https://doi.org/10.48550/arXiv.2412.15115},
  doi          = {10.48550/ARXIV.2412.15115},
  eprinttype    = {arXiv},
  eprint       = {2412.15115},
}

@article{Llama3,
  author       = {Abhimanyu Dubey and
                  Abhinav Jauhri and
                  Abhinav Pandey and
                  Abhishek Kadian and
                  Ahmad Al{-}Dahle and
                  Aiesha Letman and
                  Akhil Mathur and
                  Alan Schelten and
                  Amy Yang and
                  Angela Fan and
                  Anirudh Goyal and
                  Anthony Hartshorn and
                  Aobo Yang and
                  Archi Mitra and
                  Archie Sravankumar and
                  Artem Korenev and
                  Arthur Hinsvark and
                  Arun Rao and
                  Aston Zhang and
                  Aur{\'{e}}lien Rodriguez and
                  Austen Gregerson and
                  Ava Spataru and
                  Baptiste Rozi{\`{e}}re and
                  Bethany Biron and
                  Binh Tang and
                  Bobbie Chern and
                  Charlotte Caucheteux and
                  Chaya Nayak and
                  Chloe Bi and
                  Chris Marra and
                  Chris McConnell and
                  Christian Keller and
                  Christophe Touret and
                  Chunyang Wu and
                  Corinne Wong and
                  Cristian Canton Ferrer and
                  Cyrus Nikolaidis and
                  Damien Allonsius and
                  Daniel Song and
                  Danielle Pintz and
                  Danny Livshits and
                  David Esiobu and
                  Dhruv Choudhary and
                  Dhruv Mahajan and
                  Diego Garcia{-}Olano and
                  Diego Perino and
                  Dieuwke Hupkes and
                  Egor Lakomkin and
                  Ehab AlBadawy and
                  Elina Lobanova and
                  Emily Dinan and
                  Eric Michael Smith and
                  Filip Radenovic and
                  Frank Zhang and
                  Gabriel Synnaeve and
                  Gabrielle Lee and
                  Georgia Lewis Anderson and
                  Graeme Nail and
                  Gr{\'{e}}goire Mialon and
                  Guan Pang and
                  Guillem Cucurell and
                  Hailey Nguyen and
                  Hannah Korevaar and
                  Hu Xu and
                  Hugo Touvron and
                  Iliyan Zarov and
                  Imanol Arrieta Ibarra and
                  Isabel M. Kloumann and
                  Ishan Misra and
                  Ivan Evtimov and
                  Jade Copet and
                  Jaewon Lee and
                  Jan Geffert and
                  Jana Vranes and
                  Jason Park and
                  Jay Mahadeokar and
                  Jeet Shah and
                  Jelmer van der Linde and
                  Jennifer Billock and
                  Jenny Hong and
                  Jenya Lee and
                  Jeremy Fu and
                  Jianfeng Chi and
                  Jianyu Huang and
                  Jiawen Liu and
                  Jie Wang and
                  Jiecao Yu and
                  Joanna Bitton and
                  Joe Spisak and
                  Jongsoo Park and
                  Joseph Rocca and
                  Joshua Johnstun and
                  Joshua Saxe and
                  Junteng Jia and
                  Kalyan Vasuden Alwala and
                  Kartikeya Upasani and
                  Kate Plawiak and
                  Ke Li and
                  Kenneth Heafield and
                  Kevin Stone and
                  et al.},
  title        = {The Llama 3 Herd of Models},
  journal      = {arXiv},
  volume       = {2407.21783},
  year         = {2024},
  url          = {https://doi.org/10.48550/arXiv.2407.21783},
  eprinttype    = {arXiv},
  eprint       = {2407.21783},
}

@misc{smollm2,
      title={SmolLM2: When Smol Goes Big -- Data-Centric Training of a Small Language Model}, 
      author={Loubna Ben Allal and Anton Lozhkov and Elie Bakouch and Gabriel Martín Blázquez and Guilherme Penedo and Lewis Tunstall and Andrés Marafioti and Hynek Kydlíček and Agustín Piqueres Lajarín and Vaibhav Srivastav and Joshua Lochner and Caleb Fahlgren and Xuan-Son Nguyen and Clémentine Fourrier and Ben Burtenshaw and Hugo Larcher and Haojun Zhao and Cyril Zakka and Mathieu Morlon and Colin Raffel and Leandro von Werra and Thomas Wolf},
      year={2025},
      eprint={2502.02737},
      archivePrefix={arXiv},
      primaryClass={cs.CL},
      url={https://arxiv.org/abs/2502.02737}, 
}

@inproceedings{LoRA,
  author       = {Edward J. Hu and
                  Yelong Shen and
                  Phillip Wallis and
                  Zeyuan Allen{-}Zhu and
                  Yuanzhi Li and
                  Shean Wang and
                  Lu Wang and
                  Weizhu Chen},
  title        = {LoRA: Low-Rank Adaptation of Large Language Models},
  booktitle    = {The Tenth International Conference on Learning Representations, {ICLR}
                  2022, Virtual Event, April 25-29, 2022},
  publisher    = {OpenReview.net},
  year         = {2022},
  url          = {https://openreview.net/forum?id=nZeVKeeFYf9},
}

@inproceedings{LongtailKnowledge,
  author       = {Nikhil Kandpal and
                  Haikang Deng and
                  Adam Roberts and
                  Eric Wallace and
                  Colin Raffel},
  editor       = {Andreas Krause and
                  Emma Brunskill and
                  Kyunghyun Cho and
                  Barbara Engelhardt and
                  Sivan Sabato and
                  Jonathan Scarlett},
  title        = {Large Language Models Struggle to Learn Long-Tail Knowledge},
  booktitle    = {International Conference on Machine Learning, {ICML} 2023, 23-29 July
                  2023, Honolulu, Hawaii, {USA}},
  series       = {Proceedings of Machine Learning Research},
  volume       = {202},
  pages        = {15696--15707},
  publisher    = {{PMLR}},
  year         = {2023},
  url          = {https://proceedings.mlr.press/v202/kandpal23a.html},
}

@inproceedings{MobileLLM,
  author       = {Zechun Liu and
                  Changsheng Zhao and
                  Forrest N. Iandola and
                  Chen Lai and
                  Yuandong Tian and
                  Igor Fedorov and
                  Yunyang Xiong and
                  Ernie Chang and
                  Yangyang Shi and
                  Raghuraman Krishnamoorthi and
                  Liangzhen Lai and
                  Vikas Chandra},
  title        = {MobileLLM: Optimizing Sub-billion Parameter Language Models for On-Device
                  Use Cases},
  booktitle    = {Forty-first International Conference on Machine Learning, {ICML} 2024,
                  Vienna, Austria, July 21-27, 2024},
  publisher    = {OpenReview.net},
  year         = {2024},
  url          = {https://openreview.net/forum?id=EIGbXbxcUQ},
}

@article{FlanT5,
  author       = {Hyung Won Chung and
                  Le Hou and
                  Shayne Longpre and
                  Barret Zoph and
                  Yi Tay and
                  William Fedus and
                  Yunxuan Li and
                  Xuezhi Wang and
                  Mostafa Dehghani and
                  Siddhartha Brahma and
                  Albert Webson and
                  Shixiang Shane Gu and
                  Zhuyun Dai and
                  Mirac Suzgun and
                  Xinyun Chen and
                  Aakanksha Chowdhery and
                  Alex Castro{-}Ros and
                  Marie Pellat and
                  Kevin Robinson and
                  Dasha Valter and
                  Sharan Narang and
                  Gaurav Mishra and
                  Adams Yu and
                  Vincent Y. Zhao and
                  Yanping Huang and
                  Andrew M. Dai and
                  Hongkun Yu and
                  Slav Petrov and
                  Ed H. Chi and
                  Jeff Dean and
                  Jacob Devlin and
                  Adam Roberts and
                  Denny Zhou and
                  Quoc V. Le and
                  Jason Wei},
  title        = {Scaling Instruction-Finetuned Language Models},
  journal      = {J. Mach. Learn. Res.},
  volume       = {25},
  pages        = {70:1--70:53},
  year         = {2024},
  url          = {https://jmlr.org/papers/v25/23-0870.html},
}

@inproceedings{InstructGPT,
  author       = {Long Ouyang and
                  Jeffrey Wu and
                  Xu Jiang and
                  Diogo Almeida and
                  Carroll L. Wainwright and
                  Pamela Mishkin and
                  Chong Zhang and
                  Sandhini Agarwal and
                  Katarina Slama and
                  Alex Ray and
                  John Schulman and
                  Jacob Hilton and
                  Fraser Kelton and
                  Luke Miller and
                  Maddie Simens and
                  Amanda Askell and
                  Peter Welinder and
                  Paul F. Christiano and
                  Jan Leike and
                  Ryan Lowe},
  editor       = {Sanmi Koyejo and
                  S. Mohamed and
                  A. Agarwal and
                  Danielle Belgrave and
                  K. Cho and
                  A. Oh},
  title        = {Training language models to follow instructions with human feedback},
  booktitle    = {Advances in Neural Information Processing Systems 35: Annual Conference
                  on Neural Information Processing Systems 2022, NeurIPS 2022, New Orleans,
                  LA, USA, November 28 - December 9, 2022},
  year         = {2022},
  url          = {http://papers.nips.cc/paper\_files/paper/2022/hash/b1efde53be364a73914f58805a001731-Abstract-Conference.html},
}

@inproceedings{PAL,
  author       = {Luyu Gao and
                  Aman Madaan and
                  Shuyan Zhou and
                  Uri Alon and
                  Pengfei Liu and
                  Yiming Yang and
                  Jamie Callan and
                  Graham Neubig},
  editor       = {Andreas Krause and
                  Emma Brunskill and
                  Kyunghyun Cho and
                  Barbara Engelhardt and
                  Sivan Sabato and
                  Jonathan Scarlett},
  title        = {{PAL:} Program-aided Language Models},
  booktitle    = {International Conference on Machine Learning, {ICML} 2023, 23-29 July
                  2023, Honolulu, Hawaii, {USA}},
  series       = {Proceedings of Machine Learning Research},
  volume       = {202},
  pages        = {10764--10799},
  publisher    = {{PMLR}},
  year         = {2023},
  url          = {https://proceedings.mlr.press/v202/gao23f.html}
}

@article{emergent,
  author       = {Jason Wei and
                  Yi Tay and
                  Rishi Bommasani and
                  Colin Raffel and
                  Barret Zoph and
                  Sebastian Borgeaud and
                  Dani Yogatama and
                  Maarten Bosma and
                  Denny Zhou and
                  Donald Metzler and
                  Ed H. Chi and
                  Tatsunori Hashimoto and
                  Oriol Vinyals and
                  Percy Liang and
                  Jeff Dean and
                  William Fedus},
  title        = {Emergent Abilities of Large Language Models},
  journal      = {Trans. Mach. Learn. Res.},
  volume       = {2022},
  year         = {2022},
  url          = {https://openreview.net/forum?id=yzkSU5zdwD},
}

@inproceedings{ToolLLM,
  author       = {Yujia Qin and
                  Shihao Liang and
                  Yining Ye and
                  Kunlun Zhu and
                  Lan Yan and
                  Yaxi Lu and
                  Yankai Lin and
                  Xin Cong and
                  Xiangru Tang and
                  Bill Qian and
                  Sihan Zhao and
                  Lauren Hong and
                  Runchu Tian and
                  Ruobing Xie and
                  Jie Zhou and
                  Mark Gerstein and
                  Dahai Li and
                  Zhiyuan Liu and
                  Maosong Sun},
  title        = {ToolLLM: Facilitating Large Language Models to Master 16000+ Real-world
                  APIs},
  booktitle    = {The Twelfth International Conference on Learning Representations,
                  {ICLR} 2024, Vienna, Austria, May 7-11, 2024},
  publisher    = {OpenReview.net},
  year         = {2024},
  url          = {https://openreview.net/forum?id=dHng2O0Jjr},
}

@article{TheoreticalBoN,
  author       = {Ahmad Beirami and 
                  Alekh Agarwal and
                  Jonathan Berant and
                  Alex D'Amour and
                  Jacob Eisenstein and
                  Chirag Nagpal and
                  Ananda Theertha Suresh},
  title        = {Theoretical guarantees on the best-of-n alignment policy},
  journal      = {arXiv},
  volume       = {2401.01879},
  year         = {2024},
  url          = {https://doi.org/10.48550/arXiv.2401.01879},
  eprinttype   = {arXiv},
  eprint       = {2401.01879},
}

@inproceedings{MemorizationCapacity,
  author       = {Gavin Brown and 
                  Mark Bun and 
                  Vitaly Feldman and 
                  Adam M. Smith and 
                  Kunal Talwar},
  title        ={When is memorization of irrelevant training data necessary for high-accuracy learning?},
  booktitle      ={Proceedings of the 53rd Annual ACM SIGACT Symposium on Theory of Computing, New York, NY, USA},
  year         = {2021},
  publisher    = {Association for Computing Machinery},
  url          = {https://doi.org/10.1145/3406325.3451131},
}

@misc{VersaPRM,
      title={VersaPRM: Multi-Domain Process Reward Model via Synthetic Reasoning Data}, 
      author={Thomas Zeng and Shuibai Zhang and Shutong Wu and Christian Classen and Daewon Chae and Ethan Ewer and Minjae Lee and Heeju Kim and Wonjun Kang and Jackson Kunde and Ying Fan and Jungtaek Kim and Hyung Il Koo and Kannan Ramchandran and Dimitris Papailiopoulos and Kangwook Lee},
      year={2025},
      eprint={2502.06737},
      archivePrefix={arXiv},
      primaryClass={cs.LG},
      url={https://arxiv.org/abs/2502.06737}, 
}

@misc{MMLU-Pro,
      title={MMLU-Pro: A More Robust and Challenging Multi-Task Language Understanding Benchmark}, 
      author={Yubo Wang and Xueguang Ma and Ge Zhang and Yuansheng Ni and Abhranil Chandra and Shiguang Guo and Weiming Ren and Aaran Arulraj and Xuan He and Ziyan Jiang and Tianle Li and Max Ku and Kai Wang and Alex Zhuang and Rongqi Fan and Xiang Yue and Wenhu Chen},
      year={2024},
      eprint={2406.01574},
      archivePrefix={arXiv},
      primaryClass={cs.CL},
      url={https://arxiv.org/abs/2406.01574}, 
}

@inproceedings{SAFE,
  author       = {Jerry Wei and
                  Chengrun Yang and
                  Xinying Song and
                  Yifeng Lu and
                  Nathan Hu and
                  Jie Huang and
                  Dustin Tran and
                  Daiyi Peng and
                  Ruibo Liu and
                  Da Huang and
                  Cosmo Du and
                  Quoc V. Le},
  editor       = {Amir Globersons and
                  Lester Mackey and
                  Danielle Belgrave and
                  Angela Fan and
                  Ulrich Paquet and
                  Jakub M. Tomczak and
                  Cheng Zhang},
  title        = {Long-form factuality in large language models},
  booktitle    = {Advances in Neural Information Processing Systems 38: Annual Conference
                  on Neural Information Processing Systems 2024, NeurIPS 2024, Vancouver,
                  BC, Canada, December 10 - 15, 2024},
  year         = {2024},
  url          = {http://papers.nips.cc/paper\_files/paper/2024/hash/937ae0e83eb08d2cb8627fe1def8c751-Abstract-Conference.html}
}

@inproceedings{ToolRM,
  author       = {Lei Li and
                  Yekun Chai and
                  Shuohuan Wang and
                  Yu Sun and
                  Hao Tian and
                  Ningyu Zhang and
                  Hua Wu},
  title        = {Tool-Augmented Reward Modeling},
  booktitle    = {The Twelfth International Conference on Learning Representations,
                  {ICLR} 2024, Vienna, Austria, May 7-11, 2024},
  publisher    = {OpenReview.net},
  year         = {2024},
  url          = {https://openreview.net/forum?id=d94x0gWTUX},
}

@inproceedings{
InferenceScalingLaw,
title={Scaling Inference Computation: Compute-Optimal Inference for Problem-Solving with Language Models},
author={Yangzhen Wu and Zhiqing Sun and Shanda Li and Sean Welleck and Yiming Yang},
booktitle={The 4th Workshop on Mathematical Reasoning and AI at NeurIPS'24},
year={2024},
url={https://openreview.net/forum?id=j7DZWSc8qu}
}

@article{sLMsurvey,
  author       = {Zhenyan Lu and
                  Xiang Li and
                  Dongqi Cai and
                  Rongjie Yi and
                  Fangming Liu and
                  Xiwen Zhang and
                  Nicholas D. Lane and
                  Mengwei Xu},
  title        = {Small Language Models: Survey, Measurements, and Insights},
  journal      = {arXiv},
  volume       = {2409.15790},
  year         = {2024},
  url          = {https://doi.org/10.48550/arXiv.2409.15790},
}

@article{FSDP,
  author       = {Yanli Zhao and
                  Andrew Gu and
                  Rohan Varma and
                  Liang Luo and
                  Chien{-}Chin Huang and
                  Min Xu and
                  Less Wright and
                  Hamid Shojanazeri and
                  Myle Ott and
                  Sam Shleifer and
                  Alban Desmaison and
                  Can Balioglu and
                  Pritam Damania and
                  Bernard Nguyen and
                  Geeta Chauhan and
                  Yuchen Hao and
                  Ajit Mathews and
                  Shen Li},
  title        = {PyTorch {FSDP:} Experiences on Scaling Fully Sharded Data Parallel},
  journal      = {Proc. {VLDB} Endow.},
  volume       = {16},
  number       = {12},
  pages        = {3848--3860},
  year         = {2023},
  url          = {https://www.vldb.org/pvldb/vol16/p3848-huang.pdf},
}

@INPROCEEDINGS{pyserini,
   author = "Jimmy Lin and Xueguang Ma and Sheng-Chieh Lin and Jheng-Hong Yang and Ronak Pradeep and Rodrigo Nogueira",
   title = "{Pyserini}: A {Python} Toolkit for Reproducible Information Retrieval Research with Sparse and Dense Representations",
   booktitle = "Proceedings of the 44th Annual International ACM SIGIR Conference on Research and Development in Information Retrieval (SIGIR 2021)",
   year = 2021,
   pages = "2356--2362",
}

@article{GenerativeRewardModel,
  author       = {Dakota Mahan and
                  Duy Phung and
                  Rafael Rafailov and
                  Chase Blagden and
                  Nathan Lile and
                  Louis Castricato and
                  Jan{-}Philipp Fr{\"{a}}nken and
                  Chelsea Finn and
                  Alon Albalak},
  title        = {Generative Reward Models},
  journal      = {arXiv},
  volume       = {2410.12832},
  year         = {2024},
  url          = {https://doi.org/10.48550/arXiv.2410.12832},
}

@misc{START,
      title={START: Self-taught Reasoner with Tools}, 
      author={Chengpeng Li and Mingfeng Xue and Zhenru Zhang and Jiaxi Yang and Beichen Zhang and Xiang Wang and Bowen Yu and Binyuan Hui and Junyang Lin and Dayiheng Liu},
      year={2025},
      eprint={2503.04625},
      archivePrefix={arXiv},
      primaryClass={cs.CL},
      url={https://arxiv.org/abs/2503.04625}, 
}

@inproceedings{PaD,
  author       = {Xuekai Zhu and
                  Biqing Qi and
                  Kaiyan Zhang and
                  Xinwei Long and
                  Zhouhan Lin and
                  Bowen Zhou},
  editor       = {Kevin Duh and
                  Helena G{\'{o}}mez{-}Adorno and
                  Steven Bethard},
  title        = {PaD: Program-aided Distillation Can Teach Small Models Reasoning Better
                  than Chain-of-thought Fine-tuning},
  booktitle    = {Proceedings of the 2024 Conference of the North American Chapter of
                  the Association for Computational Linguistics: Human Language Technologies
                  (Volume 1: Long Papers), {NAACL} 2024, Mexico City, Mexico, June 16-21,
                  2024},
  pages        = {2571--2597},
  publisher    = {Association for Computational Linguistics},
  year         = {2024},
}

@inproceedings{CRITIC,
  author       = {Zhibin Gou and
                  Zhihong Shao and
                  Yeyun Gong and
                  Yelong Shen and
                  Yujiu Yang and
                  Nan Duan and
                  Weizhu Chen},
  title        = {{CRITIC:} Large Language Models Can Self-Correct with Tool-Interactive
                  Critiquing},
  booktitle    = {The Twelfth International Conference on Learning Representations,
                  {ICLR} 2024, Vienna, Austria, May 7-11, 2024},
  publisher    = {OpenReview.net},
  year         = {2024},
  url          = {https://openreview.net/forum?id=Sx038qxjek},
}

@misc{GenRMCompute,
      title={When To Solve, When To Verify: Compute-Optimal Problem Solving and Generative Verification for LLM Reasoning}, 
      author={Nishad Singhi and Hritik Bansal and Arian Hosseini and Aditya Grover and Kai-Wei Chang and Marcus Rohrbach and Anna Rohrbach},
      year={2025},
      eprint={2504.01005},
      archivePrefix={arXiv},
      primaryClass={cs.LG},
      url={https://arxiv.org/abs/2504.01005}, 
}
\bibliographystyle{iclr2026_conference}

\appendix
\newpage
\newpage
\appendix

\section{Concept-Proof Experiment Details}
\label{appendix:concpet-proof}
We provide additional details about the proof-of-concept experiment shown in~\autoref{fig:concept} (b) of the main paper. The verification task focuses on arithmetic calculations involving randomly selected $N$ three-digit numbers, using both addition and subtraction, with $N$ ranging from 3 to 10.

For each value of $N$, we generate 500 equations with correct answers, along with another 500 equations where the output is slightly incorrect—within a $5\%$ margin of error.

We then prompt Llama-3.2-1B-Instruct~\citep{Llama3} to verify these calculations. Specifically, we use Prompt~\ref{box:poc_prompt} to make the language model to generate a step-by-step explanation in natural language.
\begin{promptbox}[label={box:poc_prompt}]{Check Calculation}
Evaluate the below calculation. Is this calculation correct? If correct, return True. Return False otherwise.

\vspace{1em}

\# Calculation: \{exp\} = \{ans\}

\vspace{1em}

If the calculation is correct, return True. If not, return False.

\vspace{1em}

Think step-by-step, and MUST output True or False at the end of your verification.
\end{promptbox}

To use the tool, we prompt LM to generate a code instead of verification in natural language using Prompt~\ref{box:poc_prompt2}.
\begin{promptbox}[label={box:poc_prompt2}]{Check Calculation with Code}

Generate a simple Python script that evaluates the correctness of a given mathematical calculation.  

\vspace{1em}

\# Calculation: \{exp\} = \{ans\}  

The script should print `The calculation is correct` if the calculation is correct, otherwise print `The calculation is incorrect`.  

\vspace{1em}

\#\#\# Constraints:

- The output must be a single Python code block without any function definition.

- The script should evaluate the expression as a boolean comparison.  

\vspace{1em}

If the evaluated result of `{exp}` matches `{ans}`, print `The calculation is correct`, otherwise print `The calculation is incorrect`.
\end{promptbox}

\section{Details of Our Method}\label{appendix:detail-method}
The choice of the component of the tool-based verification function depends on the task:
\begin{enumerate}[itemsep=0.8mm, parsep=1pt, leftmargin=*]
    \item \textbf{Mathematical reasoning task: } sLM generates executable programming code $\vc_1$, then the code interpreter executes the code and validates the correctness of computations~\citep{Toolformer, ToRA}.
    Since the code interpreter executes the code as well as outputs the verification score, we can regard the extraction of the verification score part as the identity function, i.e., $\vc_2 = \mathcal{T}(\vc)$. 
    Therefore, $f(\vx, \vy; \mathcal{T},\theta)$ for numerical reasoning tasks can be represented as:
    \begin{equation}
        f(\vx, \vy; \mathcal{T},\theta) = \vc_2 = \mathcal{T}(\vc_1), \quad\text{where}\quad \vc_1 \sim \pi(\vc \mid \vx, \vy, \mI_c; \theta).
        \label{eqn:tool-math}
    \end{equation}
    \item \textbf{Knowledge-intensive task: } The tool $\mathcal{T}$ acts as a retriever that returns a set of relevant knowledge passages $\vk$ based on the input $\vx$ and candidate response $\vy$. 
    Subsequently, sLM verifies the consistency between the retrieved knowledge $\vk$ and the claims within $\vy$.
    Since the retriever utilizes $\vx$ and $\vy$ directly as a query, we can regard the tool-calling query part as the identity function, i.e., $\vc_1=(\vx,\vy)$.
    Therefore, $f(\vx, \vy; \mathcal{T},\theta)$ for knowledge-intensive tasks can be represented as:
    \begin{equation}
        f(\vx, \vy; \mathcal{T},\theta) = \vc_2 \sim \pi\left(\vc \mid \mathcal{T}\left(\vc_1\right), \vx, \vy, \mI_f; \theta\right),  \quad\text{where}\quad \vc_1=(\vx,\vy),
        \label{eqn:tool-retrieve}
    \end{equation}
\end{enumerate}

\section{Proof for Theoretical Analysis}
\subsection{From theoretical analysis to practice.}
Our theoretical analysis illustrates two key ideas: (1) tool integration reduces the memorization burden of small language models, and (2) the two stage design improves test time scaling by enabling a more reliable filtering function. These results are based on simplified and idealized settings and are not intended to capture full practical behavior. Instead, they provide conceptual illustrations grounded in existing theoretical frameworks~\citep{MemorizationCapacity, TheoreticalBoN}. The empirical trends observed in our experiments are consistent with these theoretical intuitions, suggesting that the underlying principles extend to more complex real-world scenarios.

\subsection[Proof of Lemma~\ref{thm:direct-memorization}]%
{Proof of \autorefLemma{thm:direct-memorization}}
\label{appendix:thm1_proof}
\begin{proof}

$\mathcal{X}$ has cardinality $|\mathcal{X}|=2\cdot(M-1)^3=\Theta(M^3)$. Also, Theorem 1.1 in~\citet{MemorizationCapacity} says that any learning algorithms $\gA$ that is $\varepsilon$-close-to-optimal with sufficiently small $\varepsilon>0$ also satisfies that the mutual information between data samples and the model learned by $\gA$ given the data distribution is proportional to at least the number of data samples multiplying dimension of data dimension. Since the number of data samples' cardinality in \autoref{sec:theoretical-memorization} is $\Theta(M^3)$ and the data dimension is $1$, we can state that $I\left(X;\theta|\;P\right) \;=\; \Omega\!\left(M^{3}\right)$. This proof says that if a model directly memorizes which $(a,b,c)$ pairs map to each $c=a+b$ with near-zero error, $\theta$ must encode on the order of $\!M^3$ bits of information about $X$.
\end{proof}
\subsection{Proof of \autoref{thm:code-snippet}} \label{appendix:thm2_proof}
\begin{proof}
In this case, $\gA$ has access to the tool function as follows:
\begin{align*}
   f\left(a,b,c;\mathcal{T}\right) \;=\; \mathbf{1}_{\,a+b \;=\;c\,}.
\end{align*}
Then, $\theta=\gA(X)$ such that $f(a,b,c;\theta,\mathcal{T})=\mathbf{1}_{\,a+b \;=\;c\,}$. 
Regardless of $X$, $f$ can perfectly get label $r$, resulting in $\mathrm{err}_{q,|\mathcal{X}|}(\gA) =0$. 

In addition, since $\theta$ is determined regardless of $X$ through $\gA$, we get
\begin{align*}
    I(X;\theta\mid P) &= H(\theta \mid P) - H(\theta \mid X, P) \\
    &= H(\theta \mid P) - H(\theta \mid P) \\
    &= 0,
\end{align*}
where $H$ is the entropy.
Therefore, we can state that $I\left(X;\theta|\;P\right) \;=\; 0$.
\end{proof}

\subsection{Proof of \autoref{thm:monotonicity}} \label{appendix:thm3_proof}

We first show Best-of-$N$ accuracy first in \autorefLemma{lemma:bom_w_verifier}. Then we show the monotonicity with respect to the $p$.

\begin{lemma}[Best-of-$N$ Accuracy with Imperfect Verifier]\label{lemma:bom_w_verifier}
Let the generator output $0$ or $1$ with equal probability, i.e.,
\begin{equation*}
\pi\left(0|\vx\right) = \pi\left(1|\vx\right) = \frac{1}{2},
\end{equation*}
and let the verifier $r$ with noise level $p, q$ be defined as follows:
\begin{equation*}
r_{p, q}\left(\vx,0\right) = \begin{cases} 
0, & \text{w.p. } p, \\[1mm]
1, & \text{w.p. } 1-p,
\end{cases}
\quad
r_{p, q}\left(\vx,1\right) = \begin{cases} 
1, & \text{w.p. } q, \\[1mm]
0, & \text{w.p. } 1-q,
\end{cases}
\end{equation*}
with the condition that $q > 1-q$ and $p > 1-p$. 
Then, for $N \ge 1$, the probability that the best-of-$N$ output is ground truth label, i.e. $1$, is given by
\begin{equation}
\pi^N\left(1\mid \vx\right) \;=\; \frac{q}{q+1-p}\,\left[1-\left(\frac{1-q+p}{2}\right)^N\right] \;+\; \frac{1-q}{1-q+p}\,\left(\frac{1-q+p}{2}\right)^N\,.
\end{equation}
\end{lemma}

\begin{proof}
A single sample from the generator $\pi$ is labeled \(1\) with probability \(\frac{1}{2}\) and \(0\) with probability \(\frac{1}{2}\). 
Given the verifier $r$ with noise level $p$ and $q$, the joint probability are:
\begin{equation}
\begin{aligned}
P\left(y=1, r=1 \mid \vx \right) &= \frac{1}{2}\,q,\\
P\left(y=0, r=1 \mid \vx \right) &= \frac{1}{2}\,(1-p),\\
P\left(y=1, r=0 \mid \vx \right) &= \frac{1}{2}\,(1-q),\\
P\left(y=0, r=0 \mid \vx \right)  &= \frac{1}{2}\,p.
\end{aligned}
\label{eq:joint_probability}
\end{equation}

Then, the probability that a single sample yields a verifier score of 1 and 0 are

\begin{align*}
    p\left(r=1 \mid \vx \right) &= \sum_{i\in\{0, 1\}} p\left(y=i, r=1 \mid \vx \right) = \frac{1}{2}\,\left(1-p\right) + \frac{1}{2}\,\left(q\right), \\
    p\left(r=0 \mid \vx \right) &= 1- p\left(r=1 \mid \vx \right) = \frac{1+p-q}{2}, \\
\end{align*}
respectively. 

Define the event
\begin{align*}
A = \Bigl\{ \text{at least one Best-of-$N$ sample is } r(\mathbf{x},y)=1 \Bigr\}.
\end{align*}

Then, from \autoref{eq:joint_probability}, the probability that all \(N\) candidates yield \(r=0\) is
\begin{equation}
P(A^c)=P(r=0)^N=\left(\frac{1+p-q}{2}\right)^N,
\label{eq:a_probability}
\end{equation}
and consequently,
\begin{equation}
P(A)=1-P(A^c) = 1-\left(\frac{1+p-q}{2}\right)^N.
\label{eq:a_c_probability}
\end{equation}

Consider two cases:

\textbf{Case 1.} \textit{At least one candidate yields \(r=1\) (event \(A\) occurs).}  
In this case, the final output is chosen uniformly among the candidates with \(r=1\).
For any candidate with \(r=1\), the probability that it originated from \(y=1\) is given by
\begin{align*}
P\bigl(y=1 \mid r=1, \vx\bigr)
=\frac{P(y=1, r=1 \mid \vx)}{P(r=1 \mid \vx)}
=\frac{\frac{1}{2}q}{\frac{1}{2}(q+1-p)}
=\frac{q}{q+1-p}.
\end{align*}

\textbf{Case 2.} \textit{All candidates yield \(r=0\) (event \(A^c\) occurs).}  
In this case, the output is chosen uniformly among all \(N\) candidates. For a candidate with \(r=0\), the probability that it is \(1\) is
\begin{align*}
P\bigl(y=1 \mid r=0, \vx\bigr)
=\frac{P(y=1, r=0 \mid \vx)}{P(r=0 \mid \vx)}
=\frac{\frac{1}{2}(1-q)}{\frac{1}{2}(p+1-q)}
=\frac{1-q}{p+1-q}.
\end{align*}

By the law of total probability, the overall probability that the Best-of-\(N\) output is \(1\) is
\begin{align*}
\pi^N(1\mid \mathbf{x})
&=P(A)\cdot P\bigl(y=1 \mid r=1, \vx\bigr) + P(A^c)\cdot P\bigl(y=1 \mid r=0, \vx\bigr) \\
&=P(A)\cdot\frac{q}{q+1-p} + P(A^c)\cdot\frac{1-q}{p+1-q}.
\end{align*}
From \autoref{eq:a_probability} and \autoref{eq:a_c_probability}, we have
\begin{align}
\pi^N(1\mid \mathbf{x})
=\left(1-\left(\frac{1+p-q}{2}\right)^N\right)\frac{q}{q+1-p}
+\left(\frac{1+p-q}{2}\right)^N\frac{1-q}{p+1-q}.
\label{eq:thm1_main}
\end{align}
\end{proof}

Using \autorefLemma{lemma:bom_w_verifier}, \autoref{thm:monotonicity} is proven as follows:

\begin{proof} Define the difference $\Delta = p - q,$
so that
\begin{align*}
\frac{1+p-q}{2} = \frac{1+\Delta}{2} \,,
\end{align*}
and note that
\begin{align*}
q+1-p = 1 - \Delta \quad \text{and} \quad p+1-q = 1+\Delta\,.
\end{align*}
Then, from ~\autorefLemma{lemma:bom_w_verifier}, \autoref{eq:thm1_main} is rewritten as
\begin{equation}\label{eq:f}
f(\Delta) \;=\; \left(1-\left(\frac{1+\Delta}{2}\right)^N\right)\frac{q}{1-\Delta}
+\left(\frac{1+\Delta}{2}\right)^N\frac{1-q}{1+\Delta}\,.
\end{equation}
Since \(q\) is held fixed, an increase in \(p\) corresponds to an increase in \(\Delta\).

Define
\begin{align*}
A(\Delta) = \left(\frac{1+\Delta}{2}\right)^N,
\end{align*}
so that
\begin{align*}
f(\Delta) = \bigl[1-A(\Delta)\bigr]\,\frac{q}{1-\Delta} + A(\Delta)\,\frac{1-q}{1+\Delta}\,.
\end{align*}
Differentiating \eqref{eq:f} with respect to \(\Delta\) yields
\begin{align}
f'(\Delta) = -A'(\Delta)\frac{q}{1-\Delta} + \left[1-A(\Delta)\right]\frac{q}{(1-\Delta)^2} 
+ A'(\Delta)\frac{1-q}{1+\Delta} - A(\Delta)\frac{1-q}{(1+\Delta)^2}\,,
\label{eq:f_differentiate}
\end{align}
with
\[
A'(\Delta) = \frac{N}{2}\left(\frac{1+\Delta}{2}\right)^{N-1}\,.
\]

Reformulating \autoref{eq:f_differentiate} results in
\[
f'(\Delta)
\;=\;
\frac{q}{(1-\Delta)^2}\Bigl[\,\underbrace{1 - A(\Delta)
\;-\;
\bigl(1-\Delta\bigr)\,A'(\Delta)}_{(a)}\Bigr]
\;+\;
\frac{1 - q}{(1+\Delta)^2}\Bigl[\underbrace{\bigl(1+\Delta\bigr)\,A'(\Delta)
\;-\;
A(\Delta)}_{(b)}\Bigr].
\]
\begin{enumerate}[label=(\alph*)]
\item Define $x = \tfrac{1+\Delta}{2}$.
Then,
\begin{align*}
1 - A(\Delta)
\;-\;
\bigl(1-\Delta\bigr)\,A'(\Delta) &= 1 - x^N \;-\; N\,(1-x)\,x^{N-1}.
\end{align*}
Set $g(x)=1 - x^N \;-\; N\,(1-x)\,x^{N-1}$. Then, $g(0)=1$, $g(1)=0$, and $g'(x) \le 0$ induces $g(x)>= 0$ for $x\in[0, 1]$. Since $q1 - q2 \in [-1,1]$, $x\in[0, 1]$ is satisfied. Therefore, we have
\[1 - A(\Delta)-\bigl(1-\Delta\bigr)\,A'(\Delta) > 0.\]

\item Since $ A'(\Delta) = \frac{N}{2}\biggl(\frac{1+\Delta}{2}\biggr)^{N-1} \text{and} \quad A(\Delta) = \biggl(\frac{1+\Delta}{2}\biggr)^{N}$,
we have
\begin{align*}
(1+\Delta)\,A'(\Delta) &= (1+\Delta)\,\frac{N}{2}\biggl(\frac{1+\Delta}{2}\biggr)^{N-1}\\
&= N\biggl(\frac{1+\Delta}{2}\biggr)^{N} \\
&=N \, A(\Delta).
\end{align*}
Hence, we have
\begin{align*}(1+\Delta)\,A'(\Delta)\;-\;A(\Delta)&= N\,A(\Delta)\;-\;A(\Delta) \\
&=(N-1)\,A(\Delta).
\end{align*}
For $N \ge 2$ and $\Delta > -1$, both $N-1>0$ and $A(\Delta)>0$. Therefore, 
\[\bigl(1+\Delta\bigr)\,A'(\Delta)\;-\;A(\Delta) > 0.\]
\end{enumerate}

Since (a) and (b) are non-negative, we conclude that $f'(\Delta) > 0$ for all $\Delta$ and $N \ge 2$.

Thus, for $N \ge 2$, if \(\bar{p} > \barbelow{p}\) (i.e. \(\Delta^1 > \Delta^2\)), it follows that
\[
f(\Delta^1) > f(\Delta^2)\,,
\]
or equivalently,
\[
\pi^N(1\mid \mathbf{x})\Big|_{p=\bar{p}} \;>\; \pi^N(1\mid \mathbf{x})\Big|_{p=\barbelow{p}}\,.
\]
\end{proof}

\section{Implementation Details}
\label{appendix:detail-exp}
\subsection{Model}
We use \href{https://huggingface.co/meta-llama/Meta-Llama-3-1B-Instruct}{LLaMA-3.2-1B-Instruct}~\citep{Llama3}, \href{https://huggingface.co/Qwen/Qwen2-0.5B-Instruct}{Qwen-2.5-0.5B-Instruct}~\citep{Qwen2.5}, \href{https://huggingface.co/declare-lab/SmolLM2-360M-Instruct}{SmolLM2-360M-Instruct}~\citep{smollm2} as base models for our experiments.

\subsection{Training}
\paragraph{Hyperparameters \& Setting}
As mentioned in~\autoref{sec:distillation}, we fine-tune small language models (sLMs) for each module using LoRA~\citep{LoRA}.
However, for PRM, we fine-tune the full model including the classifier head following~\citet{MathShepherd}.
Only for SmolLM2-360M-Instruct, we fine-tune the model on generation as it achieves under 10$\%$ accuracy on both GSM8K and MATH.
We organize the hyperparameter details in~\autoref{appendix:tab:hyperparams}.
We use 4 A100 40GB GPUs with FSDP~\citep{FSDP} for training.
\begin{table}[h]
\centering
\caption{Hyperparameters used in fine-tuning sLM for each component.}
\label{appendix:tab:hyperparams}
\begin{tabular}{lccc}
\toprule
\textbf{Hyperparameter} & \textbf{Verifier} & \textbf{PRM} & \textbf{ToolV} \\
\midrule
Learning rate          & $1 \times 10^{-4}$ & $1 \times 10^{-5}$ & $1 \times 10^{-4}$ \\
Batch size             & 16                 & 16               & 16 \\
Max length             & 2048               & 2048             & 2048 \\
LoRA rank              & 64                 & -                & 64 \\
LoRA $\alpha$          & 128                & -                & 128 \\
Optimizer              & AdamW              & AdamW            & AdamW \\
Training epochs        & 1                  & 3                & 3   \\
Scheduler              & Linear             & Linear           & Linear \\
\bottomrule
\end{tabular}
\end{table}

\paragraph{Dataset for Distillation}
We perform distillation using the training split of each dataset. For MMLU-Pro, we adopt the train-test split provided by~\citet{VersaPRM}. Training dataset sizes are 7473 for GSM8K, 7500 for MATH, 1284 for MMLU-Pro. 

During distillation, we prompt the teacher model—gpt-4o-mini-2024-07-18 in our experiments—to generate sequences used as supervision for training. For the generative verifier, we follow the prompt design from~\citet{GenRM}. Specifically, we generate 8 completions per problem using a temperature of 0.6, and treat these outputs as the training data.

For code generation tasks, we apply the Prompt~\ref{box:code_prompt}.
Using this prompt, we generate 4 completions per problem at a temperature of 0.6, which are then used as training samples.
For fact-checking in MMLU-Pro, we similarly generate 8 completions per problem with a temperature of 0.6, using the teacher model to construct the training dataset. We use the Prompt~\ref{box:factcheck_prompt}.
In addition, we retrieve 3 documents for fact-checking from wikipedia abstracts using BM25 implemtented in Pyserini~\citep{pyserini}.

\begin{promptbox}[label={box:code_prompt}]{Code Generation}
\texttt{SYSTEM\_PROMPT:}

Write a Python code block that verifies whether a given solution is correct based on the provided question, following these guidelines:

\vspace{1em}

- The code should be a single Python block, formatted as:

```python

CODE

```
\vspace{1em}

- The code should only print True if the solution is verified as correct. Otherwise, it should only print False if the solution is incorrect.

- Use only the following built-in modules where necessary:

\quad- `math` (for floating-point comparisons using math.isclose())
    
\quad- `sympy` (for symbolic calculations, including $\pi$ and fractions)
    
\quad- `cmath` (for complex number operations)

- For floating-point comparisons, use math.isclose() instead of $==$.

- Use `sympy.pi` for $\pi$ and `sympy.Rational` for fractions.

- Simplify all fractions and square roots without converting them to decimal values.

\vspace{1em}

\texttt{USER\_PROMPT:}

\#\#\# Input

- Question: \{question\}

- Solution: \{solution\}

\vspace{1em}

\#\#\# Output:

Return python code only.
\end{promptbox}
\begin{promptbox}[label={box:factcheck_prompt}]{Fact-checking generation}
\texttt{SYSTEM\_PROMPT:}

You are a domain expert.

\vspace{1em}

\texttt{USER\_PROMPT:}

Check the factual correctness of each statement in the provided solution to the question, using only the information available in the given document.

- Evaluate only the explicit factual claims made in the solution. Do not verify or evaluate the final conclusion or answer itself (e.g., The answer is ...).

- If a statement is factually incorrect based on the document, mark it as incorrect.

- If a statement cannot be verified using the document (i.e., the document does not confirm or deny it), treat it as not verifiable, and assume it is correct for the purpose of final verification.

\vspace{1em}

$\langle$question$\rangle$\{question\}$\langle$/question$\rangle$

\vspace{1em}

$\langle$document$\rangle$\{document\}$\langle$/document$\rangle$

\vspace{1em}

$\langle$solution$\rangle$\{solution\}$\langle$/solution$\rangle$

At the end of the fact check, provide a final summary in the following format:  
Verification: Are all statements correct? (Yes/No)? X 
(where X is either Yes or No).

If any verifiably false statement is found, output:
Verification: Are all statements correct? (Yes/No)? No

If no false statements are found (i.e., all are either correct or unverifiable), output:
Verification: Are all statements correct? (Yes/No)? Yes
\end{promptbox}

\subsection{Evaluation}
\paragraph{Hyperparameters \& Setting}
We generate $N = 64$ solutions using a temperature of 0.8.
In the case of GenRM, we follow the chain-of-thought variant proposed by~\citet{GenRM}. As in their setup, we generate $n = 8$ rationales and average the correctness scores across them, following the self-consistency method~\citep{self-consistency}, using a temperature of 0.6.
For PRM, we apply the final score aggregation approach, consistent with previous studies~\citep{MathShepherd, ScalingTestTime}.
When using ToolV for mathematical reasoning, we generate 4 code completions with a temperature of 0.6 and consider the result correct if at least one of the generated codes passes.
In knowledge-intensive reasoning with ToolV, we generate 4 rationales at the same temperature and consider the result correct only if all of them pass.
In the case of MMLU-Pro, we retrieve three documents following the training setup. Gold documents are generated from each question using GPT-4o.

\begin{table}[t]
    \centering
    \caption{Performance comparison between GenRM and ToolV + GenRM. Results are from experiments with Llama-3.2-1B-Instruct on MATH500 benchmark.}
    \vspace{-0.1in}
    \begin{tabular}{lcccc}
        \toprule
        \textbf{Method} & \textbf{Accuracy} & \textbf{Precision} & \textbf{Recall} & \textbf{F1 Score} \\
        \midrule
        GenRM & 80.91\% & 0.6153 & \bf 0.7759 & 0.6863 \\
        GenRM + ToolV & \bf 86.99\% & \bf 0.7666 & 0.7427 & \bf 0.7545 \\
        \bottomrule
    \end{tabular}
    \label{appendix:tab:verifier_accuracy}
\end{table}
\begin{table}[t]
    \centering
    \caption{Performance of LLama-3.2-1B-Instruct on the MATH500 benchmark for Python code generation, using teacher model outputs as reference (gold). We set rejection as positive label for computing precision, recall, and f1 score.}
    \vspace{-0.1in}
    \begin{tabular}{lcccc}
        \toprule
        \textbf{Model Size} & \textbf{Accuracy} & \textbf{Precision} & \textbf{Recall} & \textbf{F1 Score} \\
        \midrule
        1B & 0.7687 & 0.8720 & 0.6946 & 0.7733 \\
        3B & \bf 0.7973 & 0.8949 & \bf 0.7286 & \bf 0.8033 \\
        8B & 0.7906 & \bf 0.9207 & 0.6908 & 0.7893 \\
        \bottomrule
    \end{tabular}
    \label{appendix:tab:codegen_teacher}
\end{table}

\section{Additional experimental results}
\label{sec:appendix:additional_results}

\subsection{ToolV on multi-domain knowledge-intensive tasks}
We demonstrate that ToolV is effective in verifying solutions across a range of knowledge-intensive reasoning tasks from the subset of MMLU-Pro benchmark (Health, Economics, History domains)~\citep{MMLU-Pro}. 
We adapt ToolV to function as a \textit{fact-checker}, verifying claims in solutions without other components such as query transformation and reranker~\citep{SAFE, KARD}.
We provide experimental results on the MMLU-Pro benchmark with minimal framework in this work~\citep{MMLU-Pro}.

As shown in~\autoref{fig:main-mmlu}, ToolV outperforms the distilled PRM baseline derived from the VersaPRM~\citep{VersaPRM}. For the tool, we retrieve three documents from Wikipedia using BM25. Due to variability in document quality, performance is somewhat unstable in some cases. To explore ToolV’s upper bound, we also evaluate it using gold documents generated by GPT-4o. Results show that ToolV performance improves significantly with higher-quality documents, demonstrating its potential for multi-domain knowledge-intensive reasoning.

In~\autoref{appendix:fig:mmlu}, we plot the best-of-N results for all N values used in the experiments from~\autoref{fig:main-mmlu}.
Compared to math reasoning tasks, the plot is less clearly separated. However, ToolV + PRM generally outperforms the other baselines and clearly surpasses them even when using gold documents.

\subsection{Accuracy of generative reward model in verification generation}
In~\autoref{appendix:tab:verifier_accuracy}, we report the accuracy of GenRM with and without ToolV. The results show that ToolV significantly improves accuracy, precision, and F1 score, indicating that it effectively removes false positive cases among the solutions. The confusion matrix in~\autoref{appendix:fig:accuracy} further illustrates this trend. However, it also reveals that ToolV occasionally removes true positives, primarily due to incorrectly generated Python code.

\subsection{Accuracy of tool-based verifier in code generation}
\label{appendix:sec:python_accuracy}
In~\autoref{appendix:tab:codegen_teacher}, we present the accuracy of Python code generation, treating the teacher-generated code as the ground truth. The results show that precision is quite high—even for the 1B model, the distilled 1B ToolV is able to filter out more than 85\% of incorrect solutions among the incorrect solutions that teacher model predicted. However, the recall is low, indicating that the generated code sometimes mistakenly filters out correct solutions. 

\subsection{Data efficiency of RM and tool-based verifier in distillation}
In~\autoref{appendix:fig:datascale}, we conduct experiments to evaluate how much data is needed for each RM-based and tool-based verifier to achieve satisfactory performance. In each plot, we reduce the distillation data to 10\% and 1\% for one verifier, while keeping the other verifier fully distilled using 100\% of the dataset. The results show that ToolV maintains competitive performance even with only 10\% of the data, demonstrating its data efficiency during distillation.

\subsection{Necessity of two-stage design}
A natural question is whether ToolV alone is sufficient for verification with small models. To answer this, we conduct ablation studies that isolate the contribution of each stage. We report results in~\autoref{appendix:tab:ablation_toolv_prm} and~\autoref{appendix:tab:ablation_toolv_genrm}.

ToolV alone is a strong verifier. It already surpasses PRM on MATH500 and GenRM on GSM8K. This shows that executable consistency checks provide an effective signal for filtering incorrect solutions. However, ToolV focuses primarily on execution-based correctness such as calculation. It does not capture higher-level reasoning errors that are not executable, such as misinterpreting problem structure or producing logically inconsistent intermediate steps.

Our ablation results show that this complementary relationship leads to consistent performance gains. ToolV alone improves over PRM and GenRM, but the two stage design (ToolV plus PRM or GenRM) achieves the highest accuracy across both MATH500 and GSM8K. These findings confirm that the two stages provide different verification signals and that both are necessary for strong verification under test time scaling with small models.

\begin{figure*}
    \centering
    \includegraphics[width=\linewidth]{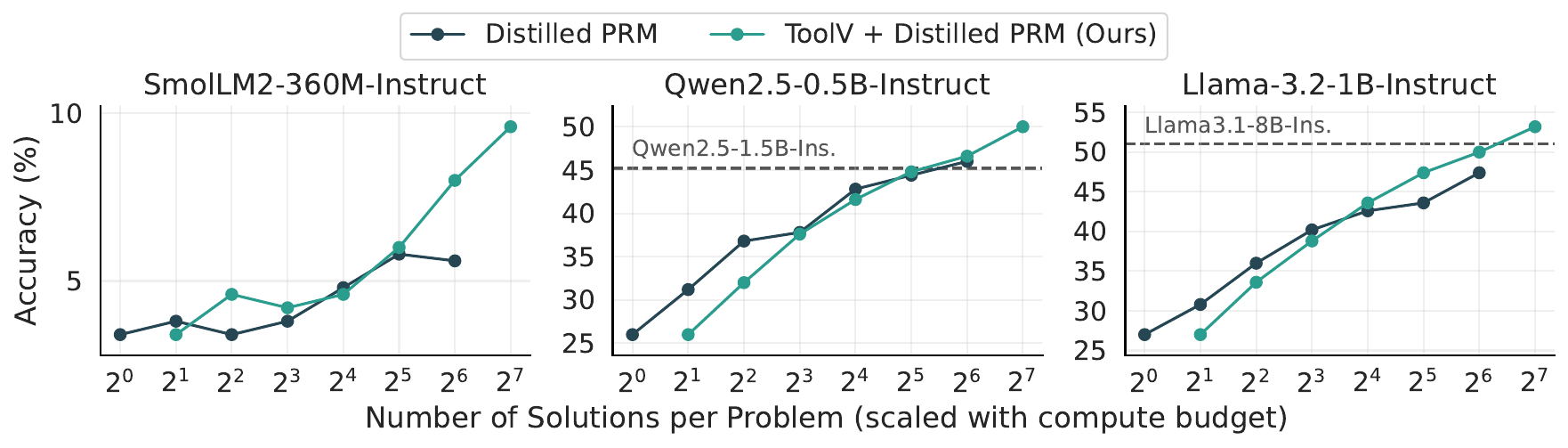}
    \vspace{-0.25in}
    \caption{\textbf{MATH500 with PRM under a scaled x-axis reflecting compute budget.} PRM only performs better at small budgets, but our method surpasses it as test-time scaling increases.} 
    \label{fig:main-prm-math-scaled}
    \vspace{-0.15in}
\end{figure*}
\subsection{Compute budget discussion with PRM}
We extend the discussion in Section~\ref{sec:discussion}.
Since PRMs require no generation, adding ToolV can be interpreted as allocating roughly twice the compute budget compared to using PRM alone. \autoref{fig:main-prm-math-scaled} presents the compute scaled comparison. PRM without ToolV performs better when the number of generated solutions is small, where the additional verification cost is not yet leveraged. As the compute budget increases, ToolV combined with PRM consistently achieves higher accuracy than PRM alone. This shows that ToolV remains effective under test-time scaling, especially for small model settings.

\begin{table*}[t]
\centering
\caption{Performance comparison of GenRM variants across different numbers of generated solutions $N$.}
\label{appendix:tab:genrm-two-stage}
\begin{tabular}{lccccccc}
\toprule
Method & $N=1$ & $N=2$ & $N=4$ & $N=8$ & $N=16$ & $N=32$ & $N=64$ \\
\midrule
One-stage GenRM   & 27.0 & 31.6 & 36.0 & 40.0 & 42.0 & 44.8 & 45.4 \\
Two-stage GenRM   & 27.0 & 29.6 & 35.0 & 39.8 & 42.0 & 44.4 & 46.0 \\
GenRM + ToolV     & 27.0 & \bf 33.8 & \bf 38.8 & \bf 43.8 & \bf 47.2 & \bf 49.4 & \bf 50.6 \\
\bottomrule
\end{tabular}
\end{table*}

\subsection{Two-stage GenRM without tool integration}
\autoref{appendix:tab:genrm-two-stage} reports the performance of two stage verification using GenRM without any tool integration. The results show that simply adding an additional verification stage does not provide the gains observed with ToolV. This confirms that the improvements come from tool integration itself rather than from increased compute through multi-stage verification.

\subsection{Exact performance of the plot}

To provide clear measurements corresponding to \autoref{fig:main-prm-math}, \autoref{fig:main-math}, and \autoref{fig:main-gsm8k}, we include tables that report the exact performance values for each model.

For \autoref{fig:main-prm-math}, the exact scores are shown in \autoref{appendix:fig:math500:smollm:prm}, \autoref{appendix:fig:math500:qwen:prm}, and \autoref{appendix:fig:math500:llama:prm} for SmolLM2-360M-Instruct, Qwen2.5-0.5B-Instruct, and Llama-3.2-1B-Instruct, respectively.

For \autoref{fig:main-math}, the corresponding tables are provided in \autoref{appendix:fig:math500:smollm:genrm}, \autoref{appendix:fig:math500:qwen:genrm}, and \autoref{appendix:fig:math500:llama:genrm}.

For \autoref{fig:main-gsm8k}, detailed results appear in \autoref{appendix:fig:gsm8k:smollm:genrm}, \autoref{appendix:fig:gsm8k:qwen:genrm}, and \autoref{appendix:fig:gsm8k:llama:genrm}.

\begin{figure*}[t]
    \centering
    \includegraphics[width=\linewidth]{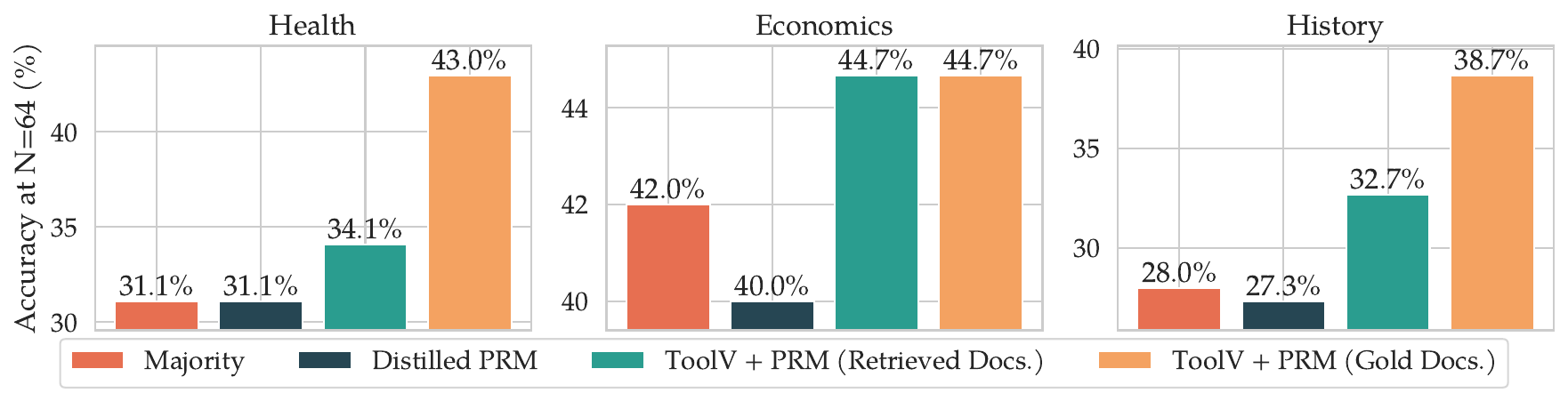}
    \vspace{-0.25in}
    \caption{\textbf{MMLU-Pro with PRM.} Weighted Best-of-N ($N=64$) performance of Llama-3.2-1B-Instruct on three knowledge-intensive domains, illustrating the effect of different document sources in ToolV + Distilled PRM (retrieved and gold documents). ToolV extends beyond math, improving PRM on multi-domain knowledge-intensive reasoning tasks.} 
    \label{fig:main-mmlu}
    \vspace{-0.1in}
\end{figure*}
\begin{figure*}[t!]
    \centering
    \includegraphics[width=\linewidth]{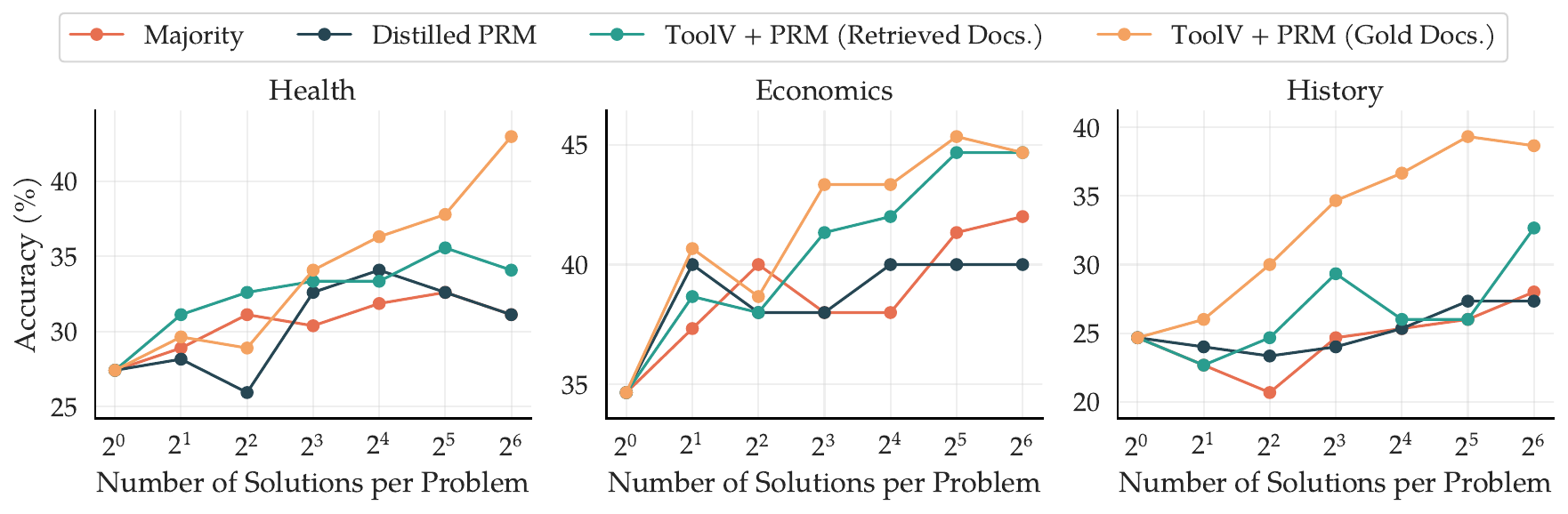}
    \vspace{-0.25in}
    \caption{\textbf{MMLU-Pro with PRM (Line Plot).} Weighted Best-of-N performance of Llama-3.2-1B-Instruct on three knowledge-intensive domains from MMLU-Pro.} 
    \label{appendix:fig:mmlu}
    \vspace{-0.1in}
\end{figure*}

\begin{figure}[t!]
    \centering

    \begin{subfigure}[b]{0.4\textwidth}
        \includegraphics[width=\textwidth]{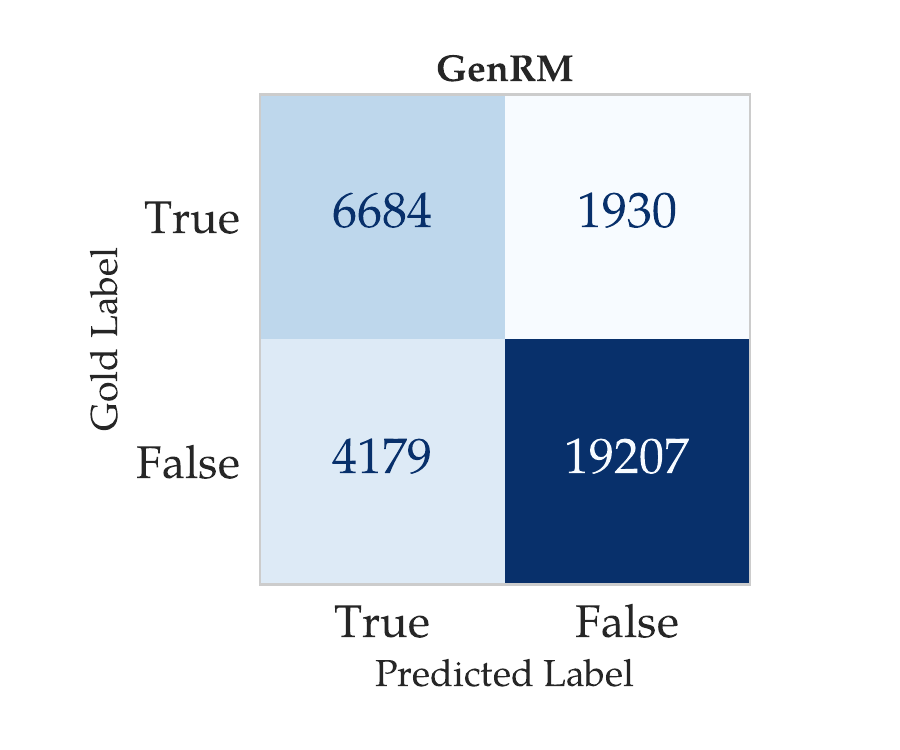}
    \end{subfigure}
    \hfill
    \begin{subfigure}[b]{0.4\textwidth}
        \includegraphics[width=\textwidth]{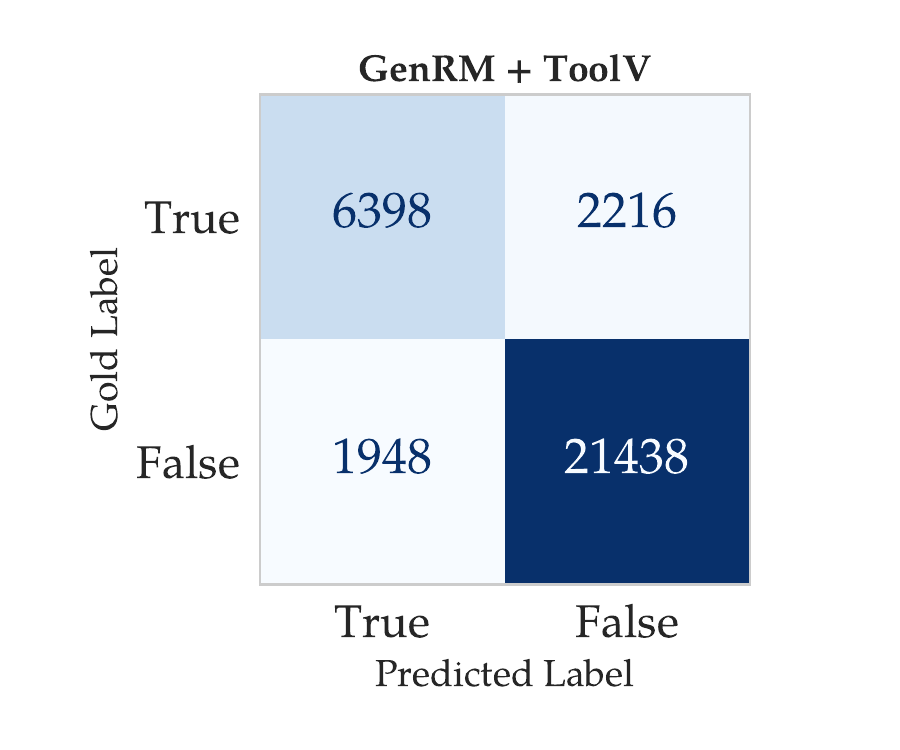}
    \end{subfigure}
    \vspace{-0.2in}
    \caption{Confusion matrix of verification results from GenRM and GenRM + ToolV, where True denotes the correct solution. This result indicates ToolV improves the performance on removing false positive cases. Results are from experiments with Llama-3.2-1B-Instruct on MATH500 benchmark.}
    \label{appendix:fig:accuracy}
\end{figure}
\begin{figure*}[t!]
    \centering
    \includegraphics[width=0.4\linewidth]{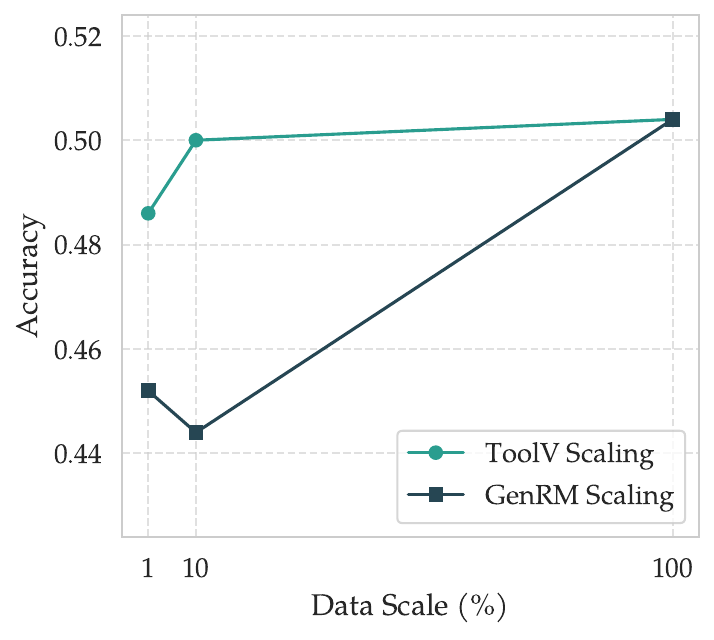}
    \vspace{-0.1in}
    \caption{\textbf{Data-scale experiment.} Performance comparison with varying distillation data sizes. In each plot, one verifier is distilled with 10\% or 1\% of data, while the other uses the full dataset. ToolV remains competitive even with only 10\% of data, highlighting its data efficiency. Results are from experiments with Llama-3.2-1B-Instruct on MATH500.}  
    \label{appendix:fig:datascale}
    \vspace{-0.1in}
\end{figure*}

\begin{table}[t]
\centering
\begin{minipage}{0.48\linewidth}
\centering
\begin{tabular}{l c}
\toprule
\textbf{MATH500} & \textbf{Accuracy} \\
\midrule
PRM & 47.4 \\
ToolV & 52.4 \\
ToolV + PRM & 53.2 \\
\bottomrule
\end{tabular}
\caption{Ablation on ToolV with PRM verifier with Llama-3.2-1B-Instruct.}
\label{appendix:tab:ablation_toolv_prm}
\end{minipage}
\hfill
\begin{minipage}{0.48\linewidth}
\centering
\begin{tabular}{l c}
\toprule
\textbf{GSM8K} & \textbf{Accuracy} \\
\midrule
GenRM & 71.87 \\
ToolV & 73.62 \\
ToolV + GenRM & 74.60 \\
\bottomrule
\end{tabular}
\caption{Ablation on ToolV with GenRM verifier with Llama-3.2-1B-Instruct.}
\label{appendix:tab:ablation_toolv_genrm}
\end{minipage}
\end{table}

\begin{table}[t]
\centering
\caption{Weighted Best-of-N accuracy on MATH500 for the SmolLM2-360M-Instruct with PRM verification. ToolV provides clear improvements over PRM and majority voting.}
\label{appendix:fig:math500:smollm:prm}
\begin{tabular}{lccccccc}
\toprule
Method & 1 & 2 & 4 & 8 & 16 & 32 & 64 \\
\midrule
Majority & 3.4 & 3.4 & 3.4 & 3.8 & 5.0 & 5.8 & 6.2 \\
PRM Verifier & 3.4 & 3.8 & 3.4 & 3.8 & 4.8 & 5.8 & 5.6 \\
ToolV + PRM Verifier (Ours) & 3.4 & 4.6 & 4.2 & 4.6 & 6.0 & 8.0 & 9.6 \\
\bottomrule
\end{tabular}
\end{table}

\begin{table}[t]
\centering
\caption{Weighted Best-of-N accuracy on MATH500 for the Qwen2.5-0.5B-Instruct model with PRM verification. ToolV consistently improves over PRM verification and majority voting.}
\label{appendix:fig:math500:qwen:prm}
\begin{tabular}{lccccccc}
\toprule
Method & 1 & 2 & 4 & 8 & 16 & 32 & 64 \\
\midrule
Majority & 26.0 & 26.0 & 32.2 & 37.0 & 39.2 & 41.6 & 42.8 \\
PRM Verifier & 26.0 & 31.2 & 36.8 & 37.8 & 42.8 & 44.4 & 46.0 \\
ToolV + PRM Verifier (Ours) & 26.0 & 32.0 & 37.6 & 41.6 & 44.8 & 46.6 & 50.0 \\
\bottomrule
\end{tabular}
\end{table}

\begin{table}[t]
\centering
\caption{Weighted Best-of-N accuracy on MATH500 for the Llama-3.2-1B-Instruct model with PRM verification. ToolV provides clear gains over PRM based verification and majority voting methods.}
\label{appendix:fig:math500:llama:prm}
\begin{tabular}{lccccccc}
\toprule
Method & 1 & 2 & 4 & 8 & 16 & 32 & 64 \\
\midrule
Majority & 27.0 & 27.0 & 33.2 & 38.8 & 41.6 & 43.8 & 45.8 \\
PRM Verifier & 27.0 & 30.8 & 36.0 & 40.2 & 42.6 & 43.6 & 47.4 \\
ToolV + PRM Verifier (Ours) & 27.0 & 33.6 & 38.8 & 43.6 & 47.4 & 50.0 & 53.2 \\
\bottomrule
\end{tabular}
\end{table}

\begin{table}[t]
\centering
\caption{Weighted Best-of-N accuracy on MATH500 for the SmolLM2-360M-Instruct model with GenRM verification. ToolV provides consistent gains over verifier based and majority voting methods.}
\label{appendix:fig:math500:smollm:genrm}
\begin{tabular}{lccccccc}
\toprule
Method & 1 & 2 & 4 & 8 & 16 & 32 & 64 \\
\midrule
Majority & 3.2 & 3.2 & 3.2 & 3.6 & 4.8 & 5.6 & 5.8 \\
Zero shot GenRM & 3.2 & 3.2 & 3.2 & 3.6 & 4.8 & 5.6 & 5.8 \\
Distilled GenRM & 3.2 & 3.4 & 3.8 & 3.8 & 5.4 & 7.2 & 7.0 \\
ToolV + Distilled GenRM (Ours) & 3.2 & 4.0 & 4.0 & 5.2 & 6.8 & 8.4 & 10.8 \\
\bottomrule
\end{tabular}
\end{table}

\begin{table}[t]
\centering
\caption{Weighted Best-of-N accuracy on MATH500 for the Qwen2.5-0.5B-Instruct model with GenRM verification. The results show that ToolV provides clear gains over verifier based and majority voting methods by filtering calculation errors that generative verification cannot resolve.}
\label{appendix:fig:math500:qwen:genrm}
\begin{tabular}{lccccccc}
\toprule
Method & 1 & 2 & 4 & 8 & 16 & 32 & 64 \\
\midrule
Majority & 26.0 & 26.0 & 32.2 & 37.0 & 39.2 & 41.6 & 42.8 \\
Zero shot GenRM & 26.0 & 27.6 & 31.0 & 37.4 & 39.8 & 42.6 & 43.0 \\
Distilled GenRM & 26.0 & 31.6 & 34.4 & 38.4 & 40.8 & 43.8 & 44.6 \\
ToolV + Distilled GenRM (Ours) & 26.0 & 32.0 & 36.8 & 40.8 & 44.0 & 45.0 & 47.2 \\
\bottomrule
\end{tabular}
\end{table}

\begin{table}[t]
\centering
\caption{Weighted Best-of-N accuracy on MATH500 for the Llama-3.2-1B-Instruct model with GenRM verification. The results show that ToolV provides clear gains over verifier based and majority voting methods by filtering calculation errors that generative verification cannot resolve.}
\label{appendix:fig:math500:llama:genrm}
\begin{tabular}{lccccccc}
\toprule
Method & 1 & 2 & 4 & 8 & 16 & 32 & 64 \\
\midrule
Majority & 27.0 & 27.0 & 33.2 & 38.8 & 41.6 & 43.8 & 45.8 \\
Zero-shot GenRM & 27.0 & 28.0 & 33.8 & 38.8 & 40.6 & 44.4 & 46.0 \\
Distilled GenRM & 27.0 & 31.6 & 36.0 & 40.0 & 42.0 & 44.8 & 45.4 \\
ToolV + Distilled GenRM (Ours) & 27.0 & 33.8 & 38.8 & 43.8 & 47.2 & 49.4 & 50.6 \\
\bottomrule
\end{tabular}
\end{table}

\begin{table}[t]
\centering
\caption{Weighted Best-of-N accuracy on GSM8K for the SmolLM2-360M-Instruct model with GenRM verification. ToolV provides clear gains over verifier based and majority voting methods.}
\label{appendix:fig:gsm8k:smollm:genrm}
\begin{tabular}{lccccccc}
\toprule
Method & 1 & 2 & 4 & 8 & 16 & 32 & 64 \\
\midrule
Majority & 10.92 & 10.92 & 13.57 & 18.35 & 23.43 & 26.69 & 28.51 \\
Zero shot GenRM & 10.92 & 10.92 & 13.57 & 18.35 & 23.43 & 26.69 & 28.51 \\
Distilled GenRM & 10.92 & 14.94 & 18.80 & 23.05 & 26.69 & 29.72 & 31.39 \\
ToolV + Distilled GenRM (Ours) & 10.92 & 15.31 & 20.92 & 26.54 & 32.75 & 36.85 & 39.88 \\
\bottomrule
\end{tabular}
\end{table}

\begin{table}[t]
\centering
\caption{Weighted Best-of-N accuracy on GSM8K for the Qwen2.5-0.5B-Instruct model with GenRM verification. ToolV consistently outperforms verifier based and majority voting methods.}
\label{appendix:fig:gsm8k:qwen:genrm}
\begin{tabular}{lccccccc}
\toprule
Method & 1 & 2 & 4 & 8 & 16 & 32 & 64 \\
\midrule
Majority & 44.28 & 44.28 & 50.49 & 56.10 & 59.14 & 60.80 & 61.71 \\
Zero shot GenRM & 44.28 & 43.44 & 50.11 & 56.63 & 59.82 & 61.18 & 62.62 \\
Distilled GenRM & 44.28 & 49.89 & 55.65 & 59.59 & 62.70 & 64.52 & 66.03 \\
ToolV + Distilled GenRM (Ours) & 44.28 & 50.87 & 56.56 & 62.62 & 64.90 & 66.57 & 68.31 \\
\bottomrule
\end{tabular}
\end{table}

\begin{table}[t]
\centering
\caption{Weighted Best-of-N accuracy on GSM8K for the Llama-3.2-1B-Instruct model with GenRM verification. ToolV shows consistent improvements over verifier based and majority voting methods.}
\label{appendix:fig:gsm8k:llama:genrm}
\begin{tabular}{lccccccc}
\toprule
Method & 1 & 2 & 4 & 8 & 16 & 32 & 64 \\
\midrule
Majority & 49.96 & 49.96 & 58.15 & 64.44 & 67.17 & 68.46 & 69.45 \\
Zero shot GenRM & 49.96 & 51.18 & 58.76 & 64.75 & 67.48 & 68.99 & 70.58 \\
Distilled GenRM & 49.96 & 57.09 & 61.94 & 66.94 & 69.67 & 71.11 & 71.87 \\
ToolV + Distilled GenRM (Ours) & 49.96 & 58.15 & 64.29 & 68.23 & 71.04 & 73.62 & 74.60 \\
\bottomrule
\end{tabular}
\end{table}

\section{Case analysis}
\label{appendix:sec:case_analysis}
In this section, we present cases where \textbf{ToolV} either enables or fails the self-verification process.  
All examples are taken from level 4 problems in \textbf{MATH500}.

In Example~\ref{ex:success}, the solution is incorrect—it computes \( 37 \times 2 \) as \( 374 \), which is wrong. 
However, the verifier (\textbf{GenRM-CoT}) fails to detect this error and incorrectly marks the step as correct, ultimately concluding that the entire solution is valid.  
In this case, \textbf{ToolV} implements a program that checks the correctness of the solution by solving the equation using the \texttt{SymPy} library.  
As expected, the result of the check is \texttt{False}, successfully identifying and removing the incorrect solution.

Occasionally, \textbf{ToolV} may produce incorrect Python code, as discussed in~\autoref{appendix:sec:python_accuracy}.  
To illustrate such a case, we present a failure example in Example~\ref{ex:fail_part1}.  
Here, the solution is actually correct, and \textbf{GenRM} correctly verifies it.  
However, \textbf{ToolV} generates code that leads to an incorrect conclusion.  
Although the code appears reasonable, the comparison \texttt{sp.simplify(result) == sp.simplify(expected\_result)} might return \texttt{False}, since symbolic expressions can differ in form even when they are mathematically equivalent.  
To properly compare equations, the code should instead use \texttt{sp.Eq(...)} which is more reliable for symbolic equality checks.

\section{LLM Usage}
We used large language models (LLMs) solely as a writing assistant, for improving grammar and clarity of the paper. No part of the research ideation, experimental design, or analysis relied on LLMs.

\clearpage

\begin{examplebox}[label={ex:success}]{Success case 1 of ToolV}
\textbf{Problem:} Half the value of $3x-9$ is $x+37$. What is the value of $x$?

\rule{\linewidth}{0.4pt}

\textbf{Solution:} \#\# Step 1: Translate the given information into an equation

We are given that half the value of $3x-9$ is $x+37$. This can be translated into the equation $\frac{3x-9}{2} = x+37$.

\#\# Step 2: Multiply both sides of the equation by 2 to clear the fraction

\textcolor{red}{Multiplying both sides of the equation by 2 to clear the fraction gives us $3x-9 = 2x + 374$.}

\#\# Step 3: Subtract $2x$ from both sides of the equation to isolate the variable term on one side

Subtracting $2x$ from both sides of the equation gives us $x - 9 = 374$.

\#\# Step 4: Add 9 to both sides of the equation to further isolate the variable term

Adding 9 to both sides of the equation gives us $x = 383$.

The final answer is: $\boxed{383}$

\rule{\linewidth}{0.4pt}

\textbf{GenRM-CoT:} Let's verify the provided solution step by step.

\#\# Step 1: Translate the given information into an equation

The solution states that half the value of \(3x - 9\) is \(x + 37\). The translation into the equation \(\frac{3x - 9}{2} = x + 37\) is correct. 

**Verification:** This step is correct.

\#\# Step 2: Multiply both sides of the equation by 2 to clear the fraction

The solution claims that multiplying both sides by 2 gives \(3x - 9 = 2x + 374\). \textbf{\textcolor{red}{This is accurate since multiplying both sides by 2 eliminates the fraction.}}

\textbf{\textcolor{red}{**Verification:** This step is correct.}}

\#\# Step 3: Subtract \(2x\) from both sides of the equation to isolate the variable term on one side

The solution states that subtracting \(2x\) from both sides results in \(x - 9 = 374\). This is also correct.

**Verification:** This step is correct.

\#\# Step 4: Add 9 to both sides of the equation to further isolate the variable term

The solution states that adding 9 to both sides gives \(x = 383\). This is accurate as well, as adding 9 to \(x - 9\) gives \(x + 9 = 374\) and thus \(x = 374 - 9\).

**Verification:** This step is correct.

The final answer is given as \( \boxed{383} \). 

\#\#\# Final Verification

All steps have been verified to be correct, and the final answer is consistent with the calculations.

Verification: Is the answer correct (Yes/No)? \textcolor{red}{Yes}

\rule{\linewidth}{0.4pt}

\textbf{ToolV}: \texttt{import sympy as sp}

\vspace{1em}

\# Define the variable

\texttt{x = sp.symbols(`x')}

\vspace{1em}

\# Given equation

\texttt{equation = sp.Eq((3*x - 9) / 2, x + 37)}

\vspace{1em}

\# Solve the equation

\texttt{solution = sp.solve(equation, x)[0]}

\vspace{1em}

\# Verify the solution

\texttt{is\_correct = sp.simplify(solution) == 383}

\texttt{print(is\_correct)}

\textbf{Interpreter output}: \textcolor{blue}{False}

\end{examplebox}
\begin{examplebox}[label={ex:fail_part1}]{Failure case 1 of ToolV}
\textbf{Problem:} The operation $\&$ is defined for positive integers $a$ and $b$ as $a \& b = \displaystyle\frac{\sqrt{a b + a}}{\sqrt{a b - b}}$. What is the value of $9 \& 2$? Express your answer as a common fraction in simplest radical form.

\rule{\linewidth}{0.4pt}

\textbf{Solution:} \#\# Step 1: Substitute the given values of $a$ and $b$ into the definition of $\&$.
We have $a = 9$ and $b = 2$, so we can substitute these values into the definition of $\&$ as follows:

$$9 \& 2 = \displaystyle\frac{\sqrt{(9)(2) + 9}}{\sqrt{(9)(2) - 2}}$$

\#\# Step 2: Simplify the expressions inside the square roots.
We can simplify the expressions inside the square roots as follows:

$$9 \& 2 = \displaystyle\frac{\sqrt{18 + 9}}{\sqrt{18 - 2}} = \displaystyle\frac{\sqrt{27}}{\sqrt{16}}$$

\#\# Step 3: Further simplify the expressions by finding square roots of perfect squares inside the fraction.
We can simplify the expressions by finding square roots of perfect squares inside the fraction as follows:

$$9 \& 2 = \displaystyle\frac{3\sqrt{3}}{4}$$

The final answer is: $\boxed{\frac{3\sqrt{3}}{4}}$

\rule{\linewidth}{0.4pt}

\textbf{GenRM-CoT:} Let's verify the solution step by step.

\#\# Step 1: Substituting values
The solution correctly substitutes \( a = 9 \) and \( b = 2 \) into the expression for \( 9 \& 2 \):
\[
9 \& 2 = \frac{\sqrt{(9)(2) + 9}}{\sqrt{(9)(2) - 2}}.
\]
This step is correct.

\#\# Step 2: Simplifying inside the square roots
The next part of the solution simplifies the expressions inside the square roots:
\[
9 \& 2 = \frac{\sqrt{18 + 9}}{\sqrt{18 - 2}} = \frac{\sqrt{27}}{\sqrt{16}}.
\]
This simplification is also correct.

\#\# Step 3: Finding square roots
The solution proceeds to find the square roots:
\[
\frac{3\sqrt{3}}{4}.
\]
This simplification is correct as \( \sqrt{27} = 3\sqrt{3} \) and \( \sqrt{16} = 4 \).

\#\# Final Answer
The final answer given is:
\[
\boxed{\frac{3\sqrt{3}}{4}}.
\]
This matches the expected answer provided in the verification section.

In summary, each step in the solution was executed correctly and led to the correct final answer.

Verification: Is the answer correct (Yes/No)? \textbf{\textcolor{blue}{Yes}}
\end{examplebox}
\begin{examplebox}[label={ex:fail:part2}]{Failure case 1 of ToolV (cont.)}
\textbf{ToolV:} \texttt{import sympy as sp}

\vspace{1em}

\# Define the operation

\texttt{def operation(a, b):}

\quad\quad\texttt{return sp.sqrt((a * b + a) / (a * b - b))}

\# Given values

\texttt{a = 9}

\texttt{b = 2}

\vspace{1em}

\# Calculate the result

\texttt{result = operation(a, b)}

\vspace{1em}

\# Expected result

\texttt{expected\_result = sp.Rational(3, 4) * sp.sqrt(3)}

\vspace{1em}

\# Verify if the result matches the expected result

\textcolor{red}{\texttt{is\_correct = sp.simplify(result) == sp.simplify(expected\_result)}}

\texttt{print(is\_correct)}

\textbf{Interpreter output}: \textbf{\textcolor{red}{False}}

\end{examplebox}

\end{document}